\documentclass{article}

\usepackage{nips_2016_modified}

\usepackage[utf8]{inputenc} 
\usepackage[T1]{fontenc}    
\usepackage{hyperref}       
\usepackage{url}            
\usepackage{booktabs}       
\usepackage{amsfonts}       
\usepackage{nicefrac}       
\usepackage{microtype}      

\usepackage{amsmath,amssymb,amsthm,mathrsfs,amsfonts,dsfont}
\usepackage{algorithm}
\usepackage{algorithmic}
\usepackage{graphicx} 
\usepackage{subfigure}

\newtheorem{theorem}{Theorem}[section]

\newtheorem{lemma}[theorem]{Lemma}
\newcommand{\bA}{\mathbf{A}}
\newcommand{\bB}{\mathbf{B}}
\newcommand{\bX}{\mathbf{X}}
\newcommand{\bY}{\mathbf{Y}}

\newcommand{\bI}{\mathbf{I}}
\newcommand{\bV}{\mathbf{V}}

\newcommand{\bW}{\mathbf{W}}

\newcommand{\tr}{\mathrm{tr}}

\newcommand{\Grad}{\mathrm{Grad}}

\title{Stochastic Variance Reduced Riemannian Eigensolver}

\author{
  Zhiqiang Xu \\
  Institute for Infocomm Research\\
  A*STAR, Singapore\\
  \texttt{zhiqiangxu2001@gmail.com} \\
  \And
  Yiping Ke \\
  Nanyang Technological University \\
  Singapore \\
  \texttt{ypke@ntu.edu.sg}
}

\begin{document}

\maketitle

\begin{abstract}
We study the stochastic Riemannian gradient algorithm for matrix eigen-decomposition. The state-of-the-art stochastic Riemannian algorithm requires the learning rate to 
decay to zero
and thus suffers from slow convergence and sub-optimal solutions. In this paper, we address this issue by deploying the variance reduction (VR) technique of stochastic gradient descent (SGD). The technique was originally developed to solve convex problems in the Euclidean space. We generalize it to Riemannian manifolds and realize it to solve the non-convex eigen-decomposition problem. We are the first to propose and analyze the generalization of SVRG to Riemannian manifolds.
Specifically, we propose the general variance reduction form, SVRRG, in the framework of the stochastic Riemannian gradient optimization. It's then specialized to the problem with eigensolvers and induces the SVRRG-EIGS algorithm. We provide a novel and elegant theoretical analysis on this algorithm. The theory shows that a fixed learning rate can be used in the Riemannian setting with an exponential global convergence rate guaranteed. The theoretical results make a significant improvement over existing studies, with the effectiveness empirically verified.

\end{abstract}

\section{Introduction}\label{sec.intro}
Matrix eigen-decomposition is among the core and long-standing topics in numerical computing \cite{wilkinson1988algebraic}. It plays fundamental roles in various scientific and engineering computing problems (such as numerical computation \cite{Golub:1996:MC:248979,Press:2007:NRE:1403886} and structural analysis \cite{1997IMA....92..135T}) as well as machine learning tasks (such as kernel approximation \cite{Drineas:2005:NMA:1046920.1194916}, dimensionality reduction \cite{citeulike:1154147} and spectral clustering \cite{NIPS2001_2092}). Thus far, there hasn't been many algorithms proposed for this problem. Pioneering ones include the method of power iteration \cite{Golub:1996:MC:248979} and the (block) Lanczos algorithm \cite{Parlett:1998:SEP:280490}, while randomized SVD \cite{Halko:2011:FSR:2078879.2078881} and online learning of eigenvectors \cite{DBLP:conf/icml/GarberHM15} are recently proposed. The problem can also be expressed as a quadratically constrained quadratic program (QCQP), and thus can be approached by various optimization methods, such as trace penalty minimization \cite{wen2013trace} and Riemannian optimization algorithms \cite{Edelman:1999:GAO:305353.305358,absil2008optimization, DBLP:journals/mp/WenY13}. Most of these algorithms perform the batch learning, i.e., using the entire dataset to perform the update at each step. 
This could be well addressed by designing appropriate stochastic algorithms. However, the state-of-the-art stochastic algorithm DSRG-EIGS \cite{zhiqiang2016}
requires the learning rate to repeatedly decay till vanishing in order to guarantee convergence, which results in a slow convergence of sub-linear rate.

We propose a new stochastic Riemannian algorithm that makes a significant breakthrough theoretically. It improves the state-of-the-art sub-linear convergence rate to an exponential convergence one. The algorithm is inspired by the stochastic variance reduced gradient (SVRG) optimization [12], which was originally developed to solve convex problems in the Euclidean space. We propose the general form of variance reduction, called {\bf SVRRG}, in the framework of the stochastic Riemannian gradient (SRG) optimization \cite{DBLP:journals/tac/Bonnabel13}, such that it is able to enjoy the convergence properties (e.g., almost sure local convergence) of the SRG framework.  We then get it specialized to the Riemannian eigensolver (RG-EIGS) problem so that it gives rise to our stochastic variance reduced Riemannian eigensolver, termed as {\bf SVRRG-EIGS}. Our theoretical analysis shows that SVRRG-EIGS can use a constant learning rate, thus eliminating the need of using the decaying learning rate. Moreover, it not only possesses the global convergence in expectation compared to SRG \cite{DBLP:journals/tac/Bonnabel13}, but also gains an accelerated convergence of exponential rate compared to DSRG-EIGS.
To the best of our knowledge, we are the first to propose and analyze the generalization of SVRG to Riemannian manifolds.

The rest of the paper is organized as follows. Section \ref{sec.pre} briefly reviews some preliminary knowledge on matrix eigen-decomposition, stochastic Riemannian gradient optimization and stochastic Riemannian eigensolver. Section \ref{sec.svr-rgeigs} presents our stochastic variance reduced Riemannian eigensolver algorithm, starting from establishing the general form of variance reduction for the stochastic Riemannian gradient optimization. Theoretical analysis is conducted in Section \ref{sec.ans}, followed by the empirical study of our algorithm in Section \ref{sec.exp}. Section 6 discusses related works. Finally, Section \ref{sec.con} concludes the paper.

\section{Preliminaries and Notations}\label{sec.pre}

\subsection{Matrix Eigen-decomposition} \label{sec.ed}
 The eigen-decomposition of a symmetric\footnote{The given matrix $\bA$ is assumed to be symmetric throughout the paper, i.e., $\bA^{\top}=\bA$.} matrix $\bA\in\mathbb{R}^{n\times n}$ can be written as $\bA=\mathbf{U}\mathbf{\Lambda}\mathbf{U}^{\top}$, where  $\mathbf{U}^{\top}\mathbf{U}=\mathbf{U}\mathbf{U}^{\top}=\bI$ (identity matrix), and $\mathbf{\Lambda}$ is a diagonal matrix. The $j$-th column $\mathbf{u}_{j}$ of $\mathbf{U}$ is called the eigenvector corresponding to the eigenvalue $\lambda_{j}$ ($j$-th diagonal element of $\mathbf{\Lambda}$), i.e.,  $\bA\mathbf{u}_{j}=\lambda_{j}\mathbf{u}_{j}$. Assume that $\lambda_{1}\geq\cdots\geq \lambda_{n}$, $\bV= [\mathbf{u}_{1},\cdots,\mathbf{u}_{k}]$ and $\bV_{\perp}= [\mathbf{u}_{k+1},\cdots,\mathbf{u}_{n}]$, $\mathbf{\Sigma}=\mathrm{diag}(\lambda_{1},\cdots,\lambda_{k})$ and $\mathbf{\Sigma}_{\perp}=\mathrm{diag}(\lambda_{k+1},\cdots,\lambda_{n})$.
In practice, matrix eigen-decomposition only aims at the set of top eigenvectors $\bV$. From the optimization perspective, this can be formulated as the following non-convex QCQP problem: 
\begin{eqnarray}
\max_{\bX\in\mathbb{R}^{n\times k}: \bX^{\top}\bX=\bI}(1/2)\tr(\bX^{\top}\bA\bX), \label{opt.problem}
\end{eqnarray}
where $k\ll n$ and $\tr(\cdot)$ represents the trace of a square matrix, i.e., the sum of diagonal elements of a square matrix. It can be easily verified that $\bX=\bV$ maximizes the trace at $(1/2)\sum_{i=1}^{k}\lambda_{i}$.

\subsection{Stochastic Riemannian Gradient Optimizaiton}
Given a Riemmanian manifold $\mathcal{M}$, the tangent space at a point $\bX\in\mathcal{M}$, denoted as $T_{\bX}\mathcal{M}$, is a Euclidean space that locally linearizes $\mathcal{M}$ around $\bX$ \cite{Lee2012}. One iterate of the Riemannian  gradient optimization on $\mathcal{M}$ takes the form similar to that of the Euclidean case \cite{absil2008optimization}:
\begin{eqnarray}
\bX^{(t+1)} = R_{\bX^{(t)}}(\alpha_{t+1}\xi_{\bX^{(t)}}), \label{eqn.rg-related}
\end{eqnarray}
where $\xi_{\bX^{(t)}}\in T_{\bX^{(t)}}\mathcal{M}$ is a tangent vector of $\mathcal{M}$ at $\bX^{(t)}$ and represents the search direction at the $t$-th step, $\alpha_{t+1}>0$ is the learning rate (i.e., step size), and $R_{\bX^{(t)}}(\cdot)$ represents the retraction at $\bX^{(t)}$ that maps a tangent vector $\xi\in T_{\bX^{(t)}}\mathcal{M}$ to a point on $\mathcal{M}$.
Tangent vectors that serve as search directions are generally gradient-related. The gradient of a function $f(\bX)$ on $\mathcal{M}$, denoted as $\Grad f(\bX)$, depends on the Riemannian metric, which is a family of smoothly varying inner products on tangent spaces, i.e., $\langle \xi, \eta\rangle_{\bX}$, where $\xi, \eta\in T_{\bX}\mathcal{M}$ for any $\bX\in\mathcal{M}$. The Riemannian gradient $\Grad f(\bX)\in T_{\bX}\mathcal{M}$ is the unique tangent vector that satisfies
\begin{eqnarray}
\langle \Grad f(\bX),\xi\rangle_{\bX}=Df(\bX)[\xi] \label{grad_def}
\end{eqnarray}
for any $\xi\in T_{\bX}\mathcal{M}$, where $Df(\bX)[\xi]$ represents the directional derivative of $f(\bX)$ in the tangent direction $\xi$. Setting $\xi_{\bX^{(t)}}=\Grad f(\bX^{(t)})$ in (\ref{eqn.rg-related}) leads to the Riemannian gradient (RG) ascent method:
\begin{eqnarray}
\bX^{(t+1)} = R_{\bX^{(t)}}(\alpha_{t+1}\Grad f(\bX^{(t)})). \label{eqn.rg}
\end{eqnarray}
We can also set $\xi_{\bX^{(t)}}=G(y_{t+1},\bX^{(t)})$ in (\ref{eqn.rg-related}) and induce the stochastic Riemannian gradient (SRG) ascent method \cite{DBLP:journals/tac/Bonnabel13}:
\begin{eqnarray}
\bX^{(t+1)} = R_{\bX^{(t)}}(\alpha_{t+1}G(y_{t+1},\bX^{(t)})), \label{eqn.srg}
\end{eqnarray}
where $y_{t+1}$ is an observation of the random variable $y$ at the $t$-th step that follows some distribution and satisfies $\mathbb{E}[f(y,\bX)]=f(\bX)$, and $G(y,\bX)\in T_{\bX}\mathcal{M}$ is the stochastic Riemannian gradient such that $\mathbb{E}[G(y,\bX)]=\Grad f(\bX)$. According to \cite{DBLP:journals/tac/Bonnabel13}, the SRG method possesses the almost sure (local) convergence under certain conditions, including $\sum_{t}\alpha_{t} = \infty $ and $\sum_{t}\alpha_{t}^{2} < \infty$ (the latter condition implies that $\alpha_{t}\rightarrow 0$ as $t\rightarrow \infty$).

\subsection{Stochastic Riemannian Eigensolver} \label{sec.srg-eigs}
The constraint set in problem (\ref{opt.problem}) constitutes a Stiefel manifold, $\mathrm{St}(n,k)=\{\bX\in\mathbb{R}^{n\times k}: \bX^{\top}\bX=\bI\}$, which turns (\ref{opt.problem}) into a Riemannian optimization problem: 
\begin{eqnarray}
\max_{\bX\in\mathrm{St}(n,k)}f(\bX), \label{eqn.eig-prob}
\end{eqnarray}
where $f(\bX)=\frac{1}{2}\tr(\bX^{\top}\bA\bX)$. Note that $\mathrm{St}(n,k)$ is an embedded Riemannian sub-manifold of the Euclidean space $\mathbb{R}^{n\times k}$ \cite{absil2008optimization}.
With the metric inherited from the embedding space $\mathbb{R}^{n\times k}$, i.e., $\langle\xi,\eta\rangle_{\bX}=\tr(\xi^{\top}\eta)$,
and using (\ref{grad_def}), we can get the Riemannian gradient\footnote{Due to the symmetry of $\bA$, the Riemannian gradients under Euclidean metric and canonical metric are the same \cite{DBLP:journals/mp/WenY13}. However, since the orthogonal projector used in the sequel requires the metrics for the embedded Riemannian sub-manifold and the embedding space to be the same, we choose the Euclidean metric here.} $\Grad f(\bX)\in T_{\bX}\mathrm{St}(n,k)$ as: 
\begin{eqnarray*}
\Grad f(\bX) = (\bI-\bX\bX^{\top})\bA\bX.
\end{eqnarray*}
The orthogonal projection onto $T_{\bX}\mathrm{St}(n,k)$ under this metric is given by:
\begin{eqnarray}
P_{\bX}(\zeta) = (\bI-\bX\bX^{\top})\zeta + \bX\mathrm{skew}(\bX^{\top}\zeta) \in T_{\bX}\mathrm{St}(n,k) \label{eqn.proj}
\end{eqnarray}
for any $\zeta\in T_{\bX}\mathbb{R}^{n\times k}\simeq \mathbb{R}^{n\times k}$, where 
$\mathrm{skew}(H)=(H-H^{\top})/2$.
In this paper, we use the retraction \cite{absil2008optimization} 
\begin{eqnarray}
R_{\bX}(\xi)=(\bX+\xi)(\bI + \xi^{\top}\xi)^{-1/2} \label{retraction}
\end{eqnarray}
for any $\xi\in T_{\bX}\mathrm{St}(n,k)$. The deployment of (\ref{eqn.rg}) and (\ref{eqn.srg}) here will then generate the Riemannian eigensolver (denoted as RG-EIGS) and the stochastic Riemannian eigensolver (denoted as SRG-EIGS), respectively. To the best of our knowledge, there is no existing stochastic Riemannian eigensolver that uses this retraction. The closest counterpart is the DSRG-EIGS 
that uses the Cayley transformation based retraction. However, based on the work of DSRG-EIGS, it can be shown that SRG-EIGS possesses the same theoretical properties as DSRG-EIGS, e.g., sub-linear convergence to global solutions.

\section{SVRRG-EIGS}\label{sec.svr-rgeigs}
In this section, we propose the stochastic variance reduced Riemannian gradient (SVRRG) and specialize it to the eigensolver problem.

\subsection{SVRRG}\label{sec.svr-rg}
Recall that the stochastic variance reduced gradient (SVRG) \cite{NIPS2013_4937} is built on the vanilla stochastic gradient and achieves variance reduction through constructing control variates \cite{NIPS2013_5034}. Control variates are stochastic and zero-mean, serving to augment and correct stochastic gradients towards the true gradients. Following \cite{NIPS2013_4937}, SVRG is encoded as
\begin{eqnarray}
g_{t}(\xi_{t},w^{(t-1)}) = \nabla\psi_{i_{t}}(w^{(t-1)}) - (\nabla\psi_{i_{t}}(\tilde{w})-\nabla P(\tilde{w})), \label{eqn.svrg}
\end{eqnarray}
where $\tilde{w}$ is a version of the estimated $w$ that is kept as a snapshot after every $m$ SGD steps, and $\nabla P(\tilde{w})=\frac{1}{n}\sum_{i=1}^{n}\nabla\psi_{i}(\tilde{w})$ is the full gradient at $\tilde{w}$.

Our task here is to develop the Riemannian counterpart SVRRG of SVRG. Denote the SVRRG as $\tilde{G}(y_{t+1},\bX^{(t)})$. A naive adaptation of (\ref{eqn.svrg}) to a Riemannian manifold $\mathcal{M}$ reads
\begin{eqnarray*}
\tilde{G}(y_{t+1},\bX^{(t)}) = G(y_{t+1},\bX^{(t)}) - (G(y_{t+1},\tilde{\bX})-\Grad f(\tilde{\bX})),
\end{eqnarray*}
where $G(y_{t+1},\bX^{(t)})\in T_{\bX^{(t)}}\mathcal{M}$ and $G(y_{t+1},\tilde{\bX}),\Grad f(\tilde{\bX})\in T_{\tilde{\bX}}\mathcal{M}$. However, this adaptation is not sound theoretically: the stochastic Riemannian gradient $G(y_{t+1},\bX^{(t)})$ and the control variate $G(y_{t+1},\tilde{\bX})-\Grad f(\tilde{\bX})$ reside in two different tangent spaces, 
and thus making their difference $\tilde{G}(y_{t+1},\bX^{(t)})$ not well-defined. We rectify this problem by the parallel transport \cite{absil2008optimization}, which moves tangent vectors from one point to another (accordingly from one tangent space to another) along geodesics in parallel. More specifically, we parallel transport the control variate from $\tilde{\bX}$ to $\bX^{(t)}$. For computational efficiency,  the first-order approximation, called vector transport \cite{absil2008optimization}, is used.

Vector transport of a tangent vector from point $\tilde{\bX}$ to point $\bX^{(t)}$, denoted as $\mathcal{T}_{\tilde{\bX}\rightarrow\bX^{(t)} }$, is a mapping from tangent space $T_{\tilde{\bX}}\mathcal{M}$ to tangent space $T_{\bX^{(t)}}\mathcal{M}$. When $\mathcal{M}$ is an embedded Riemannian sub-manifold of a Euclidean space, vector transport can be simply defined as \cite{absil2008optimization}:
\begin{eqnarray*}
\mathcal{T}_{\tilde{\bX}\rightarrow\bX^{(t)} }(\xi_{\tilde{\bX}})= P_{\bX^{(t)}}(\xi_{\tilde{\bX}}),
\end{eqnarray*}
where $P_{\bX^{(t)}}(\cdot)$ represents the orthogonal projector onto $T_{\bX^{(t)}}\mathcal{M}$ for the embedding Euclidean space.
With the vector transport, we obtain the well-defined SVRRG in $T_{\bX^{(t)}}\mathcal{M}$:
\begin{eqnarray*}
\tilde{G}(y_{t+1},\bX^{(t)}) = G(y_{t+1},\bX^{(t)}) - \mathcal{T}_{\tilde{\bX}\rightarrow\bX^{(t)} }(G(y_{t+1},\tilde{\bX})-\Grad f(\tilde{\bX})).
\end{eqnarray*}
We then arrive at our SVRRG method:
\begin{eqnarray}
\bX^{(t+1)} = R_{\bX^{(t)}}(\alpha_{t+1}\tilde{G}(y_{t+1},\bX^{(t)})), \label{eqn.svr-rg}
\end{eqnarray}
by setting $\xi_{\bX^{(t)}}=\tilde{G}(y_{t+1},\bX^{(t)})$ in (\ref{eqn.rg-related}). Note that the SVRRG method (\ref{eqn.svr-rg}) is naturally subsumed into the SRG method (\ref{eqn.srg}), and thus enjoys all the properties of SRG.

\begin{algorithm}
\caption{SVRRG}
\label{algo}
\begin{algorithmic}[1]
\REQUIRE Data $\bA$, initial $\tilde{\bX}^{(0)}$, learning rate $\alpha$, epoch length $m$
\vspace{0.5mm}
\FOR{$s=1,2,\cdots$}
\STATE Compute $\Grad f(\tilde{\bX}^{(s-1)})$
\STATE $\bX^{(0)}=\tilde{\bX}^{(s-1)}$
\FOR{$t=1,2,\cdots,m$}
\STATE Pick $y_{t}$ from the sample space uniformly at random
\STATE Compute $G(y_{t},\bX^{(t-1)})$ and $G(y_{t},\tilde{\bX}^{(s-1)})$
\STATE Compute $\mathcal{T}_{\tilde{\bX}^{(s-1)}\rightarrow\bX^{(t-1)} }(G(y_{t},\tilde{\bX}^{(s-1)})-\Grad f(\tilde{\bX}^{(s-1)}))$
\STATE Compute $G(y_{t},\bX^{(t-1)}) - \mathcal{T}_{\tilde{\bX}^{(s-1)}\rightarrow\bX^{(t-1)} }(G(y_{t},\tilde{\bX}^{(s-1)})-\Grad f(\tilde{\bX}^{(s-1)}))$
\STATE Compute $\bX^{(t)} = R_{\bX^{(t-1)}}(\alpha\tilde{G}(y_{t},\bX^{(t-1)}))$
\ENDFOR
\STATE $\tilde{\bX}^{(s)}=\bX^{(m)}$
\ENDFOR
\vspace{0.5mm}
\end{algorithmic}
\end{algorithm}

\subsection{SVRRG-EIGS}
With the SVRRG described above, we can now proceed to develop an effective eigensolver by specializing (\ref{eqn.eig-prob}). This new eigensolver is named SVRRG-EIGS.
The update can be written as
\begin{eqnarray}
\bX^{(t+1)} = (\bX^{(t)}+\alpha_{t+1}\tilde{G}(y_{t+1},\bX^{(t)}))(\bI + \alpha_{t+1}^{2}\tilde{G}^{\top}(y_{t+1},\bX^{(t)})\tilde{G}(y_{t+1},\bX^{(t)}))^{-1/2},
\end{eqnarray}
which can be decomposed into two substeps: $\bY^{(t+1)}\triangleq\bX^{(t)}+\alpha_{t+1}\tilde{G}(y_{t+1},\bX^{(t)})$ and $\bX^{(t+1)} = \bY^{(t+1)}(\bI + \alpha_{t+1}^{2}\tilde{G}^{\top}(y_{t+1},\bX^{(t)})\tilde{G}(y_{t+1},\bX^{(t)}))^{-1/2}$.
Intuitively, the first substep moves
along the direction $\tilde{G}(y_{t+1},\bX^{(t)})$ from the current point $\bX^{(t)}$ to the intermediate point $\bY^{(t+1)}$ in the tangent space $T_{\bX^{(t)}}\mathrm{St}(n,k)$. The second substep then
gets the intermediate point $\bY^{(t+1)}$ retracted back onto the Stiefel manifold $\mathrm{St}(n,k)$ to reach the next point $\bX^{(t+1)}$.

Let's delve into the first substep. Except for the vector transport inside $\tilde{G}(y_{t+1},\bX^{(t)})$, it looks much like an SVRG step since it works in the Euclidean tangent space. Assume that we have $\bA=\frac{1}{L}\sum_{l=1}^{L}\bA^{(l)}$, $y$ is a random variable taking values in $\{1,2,\cdots,L\}$, $\bA_{t+1} = \bA^{(y_{t+1})}$, and stochastic gradient takes the form $G(y_{t+1},\bX)=(\bI-\bX\bX^{\top})\bA_{t+1}\bX$ (i.e., sampling over data $\bA$). We can get the control variate as
\begin{eqnarray*}
G(y_{t+1},\tilde{\bX})-\Grad f(\tilde{\bX})= (\bI-\tilde{\bX}\tilde{\bX}^{\top})(\bA_{t+1}-\bA)\tilde{\bX}.
\end{eqnarray*}
By using the orthogonal projector in
(\ref{eqn.proj}), the transported control variate can be written as
\begin{eqnarray*}
&&\mathcal{T}_{\tilde{\bX}\rightarrow\bX^{(t)} }(G(y_{t+1},\tilde{\bX})-\Grad f(\tilde{\bX}))\\
&=& (\bI-\bX^{(t)}\bX^{(t)^{\top}})(\bI-\tilde{\bX}\tilde{\bX}^{\top})(\bA_{t+1}-\bA)\tilde{\bX}+ \bX^{(t)}\mathrm{skew}(\bX^{(t)^{\top}}(\bI-\tilde{\bX}\tilde{\bX}^{\top})(\bA_{t+1}-\bA)\tilde{\bX})\qquad\\
&=& (\bI-\bX^{(t)}\bX^{(t)^{\top}})(\bA_{t+1}-\bA)\tilde{\bX}-(\bI-\bX^{(t)}\bX^{(t)^{\top}})\tilde{\bX}\tilde{\bX}^{\top}(\bA_{t+1}-\bA)\tilde{\bX}+\\
&&\bX^{(t)}\mathrm{skew}(\bX^{(t)^{\top}}(\bI-\tilde{\bX}\tilde{\bX}^{\top})(\bA_{t+1}-\bA)\tilde{\bX}).
\end{eqnarray*}
Accordingly, we have the SVRRG expressed as
\begin{eqnarray*}
\tilde{G}(y_{t+1},\bX^{(t)})
&=&(\bI-\bX^{(t)}\bX^{(t)^{\top}})\bA_{t+1}\bX^{(t)} - \mathcal{T}_{\tilde{\bX}\rightarrow\bX^{(t)} }(G(y_{t+1},\tilde{\bX})-\Grad f(\tilde{\bX}))\\
&=& (\bI-\bX^{(t)}\bX^{(t)^{\top}})\bA\bX^{(t)} + (\bI-\bX^{(t)}\bX^{(t)^{\top}})(\bA_{t+1}-\bA)(\bX^{(t)}-\tilde{\bX})+\\
&&(\bI-\bX^{(t)}\bX^{(t)^{\top}})\tilde{\bX}\tilde{\bX}^{\top}(\bA_{t+1}-\bA)\tilde{\bX}- \bX^{(t)}\mathrm{skew}(\bX^{(t)^{\top}}(\bI-\tilde{\bX}\tilde{\bX}^{\top})(\bA_{t+1}-\bA)\tilde{\bX})\qquad\\
&\triangleq& \Grad f(\bX^{(t)}) + \bW^{(t+1)},
\end{eqnarray*}
where $\bW^{(t+1)}\in T_{\bX^{(t)}}\mathrm{St}(n,k)$ is a stochastic zero-mean term conditioned on $\bX^{(t)}$. Note that the factor $(\bX^{(t)}-\tilde{\bX})$ in $\bW^{(t+1)}$ might be theoretically harsh\footnote{It works well empirically.}, because an eigenspace could have distinct representations which are the same up to a $k\times k$ orthogonal matrix, and thus it is only expected that $\bX^{(t)}$ and $\tilde{\bX}$ have the same column space at convergence. The ideal replacement would be $(\textrm{col}(\bX^{(t)})-\textrm{col}(\tilde{\bX}))$ where $\textrm{col}(\cdot)$ represents the column space. Numerically it could be achieved by replacing $\tilde{\bX}$ with $\tilde{\bX}\bB^{(t)}$ where $\bB^{(t)}=\mathbf{Q}_{2}\mathbf{Q}_{1}^{\top}$ and $\bX^{(t)^{\top}}\tilde{\bX}=\mathbf{Q}_{1}\mathbf{\Omega}\mathbf{Q}_{2}^{\top}$ is the SVD of $\bX^{(t)^{\top}}\tilde{\bX}$ \cite{shamir2015fast}.

The first substep can now be rewritten as
\begin{eqnarray*}
\textrm{SVRRG-EIGS}:\; \bY^{(t+1)}=(\bX^{(t)}+\alpha_{t+1}\Grad f(\bX^{(t)})) + \alpha_{t+1}\bW^{(t+1)}.
\end{eqnarray*}
As a comparison, we can similarly decompose the update steps (\ref{eqn.rg}) and (\ref{eqn.srg}) of RG-EIGS and SRG-EIGS into two substeps and then have:
\begin{eqnarray*}
\textrm{RG-EIGS}:\; \bY^{(t+1)}&=&\bX^{(t)}+\alpha_{t+1}\Grad f(\bX^{(t)}),\\
\textrm{SRG-EIGS}:\; \bY^{(t+1)}&=&(\bX^{(t)}+\alpha_{t+1}\Grad f(\bX^{(t)})) + \alpha_{t+1}(\bI-\bX^{(t)}\bX^{(t)^{\top}})(\bA_{t+1}-\bA)\bX^{(t)}.
\end{eqnarray*}
Compared to that in
RG-EIGS, each step in both SRG-EIGS and SVRRG-EIGS amounts to taking one Riemannian gradient step in the tangent space, adding a stochastic zero-mean term in the tangent space, and then retracting back to the manifold.
However, the stochastic zero-mean term $(\bI-\bX^{(t)}\bX^{(t)^{\top}})(\bA_{t+1}-\bA)\bX^{(t)}$ in SRG-EIGS has a constant variance. Therefore it needs the learning rate $\alpha_{t}$ to decay to zero to reduce the variance and to ensure the convergence, and consequently compromises on the convergence rate. In contrast, SVRRG-EIGS keeps boosting the variance reduction of the stochastic zero-mean term $\bW^{(t+1)}$ during iterations. The variance of $\bW^{(t+1)}$ is not constant but dominated by three quantities $\|\bX^{(t)}-\tilde{\bX}\bB^{(t)}\|$, $\|(\bI-\bX^{(t)}\bX^{(t)^{\top}})\tilde{\bX}\bB^{(t)}\|$ and $\|\bX^{(t)^{\top}}(\bI-\tilde{\bX}\tilde{\bX}^{\top})\|$. These quantities repeatedly decay till vanishing in expectation, as $\bX^{(t)}$ and $\tilde{\bX}\bB^{(t)}$ are expected to get closer and closer to each other gradually. This induces a decaying variance
without the learning rate involved. Therefore, SVRRG-EIGS is able to use a fixed learning rate $\alpha_{t}=\alpha$ and achieve a much faster convergence rate.

\section{Theoretical Analysis}\label{sec.ans}
We give the main theoretical results in this section. The proofs are provided in the supplementary material.
\begin{theorem} \label{thm}
Consider a symmetric matrix $\bA\in\mathbb{R}^{n\times n}$ which can be written as $\bA=\frac{1}{L}\sum_{l=1}^{L}\bA^{(l)}$ such that $\max_{l}\|A^{(l)}\|_{2}\leq 1$. The eigen-decomposition of $\bA$ is as defined in Section \ref{sec.ed}. And the eigen-gap $\tau=\lambda_{k}-\lambda_{k+1}>0$. Then the top $k$ eigenvectors $\bV$ can be approximated to arbitrary accuracy $\varepsilon\in (0,1)$ and with any confidence level $\varphi\in (0, \frac{1}{\lceil\log_{2}(1/\varepsilon)\rceil})$ by running $T=\lceil\frac{\log(1/\varepsilon)}{\log(2/\varphi)}\rceil$ epochs of our SVRRG-EIGS algorithm, in the sense that the potential function $\mathbf{\Theta}(\tilde{\bX}^{(T)})=k-\|\bV^{\top}\tilde{\bX}^{(T)}\|_{F}^{2}\leq\varepsilon$ with probability at least $1-\lceil\log_{2}(1/\varepsilon)\rceil\varphi$, provided that the following conditions about initial iterate $\tilde{\bX}^{(0)}$, fixed learning rate $\alpha$ and epoch length $m$, are simultaneously satisfied:
\begin{eqnarray*}
\tilde{b}_{0}=k-\|\bV^{\top}\tilde{\bX}^{(0)}\|_{F}^{2}<\frac{1}{2},\quad \alpha\in(0,\min\{c_{0}\tau,\frac{c_{1}}{8c_{2}}\tau\varphi^{2}\}),\\
m\geq \frac{3\log(2/\varphi)}{c_{1}\alpha\tau},\quad c_{3}km\alpha^{2} + c_{5}k\sqrt{m\alpha^{2}\log(2/\varphi)}\leq \frac{1}{2}-\tilde{b}_{0},
\end{eqnarray*}
where the constants are positive and defined as
\begin{eqnarray*}
c_{0}&=&\min\{\frac{1}{32\sqrt{3k\tau^{2}}},\frac{1}{c_{1}\tau^{2}},\frac{-(118406+144k^{2})+\sqrt{(118406+144k^{2})^{2}+18\tau(1+24k^{2})}}{24\tau(1+24k^{2})}\},\\
c_{1}&=& \frac{2}{\tau}(\frac{1}{8}\tau-2\alpha(1+2\alpha)(1+24k^{2})-\frac{118400}{3}\alpha),\quad c_{2} = 96(k^{2}(1+2\alpha)+823),\\
c_{3}&=& 4(1+2\alpha)+192(k^{2}(1+2\alpha)+\frac{7400}{9}),\quad c_{4}=\frac{20}{1-5c_{0}\tau}+c_{0}c_{3}\tau,\quad c_{5}=\sqrt{2c_{4}}.
\end{eqnarray*}
\end{theorem}

Note that we have no loss of generality from assuming that $\max_{l}\|\bA^{(l)}\|_{2}\leq 1$ in the theorem. In fact, if $\max_{l}\|\bA^{(l)}\|_{2}\leq r$ with $r>1$ (which could be estimated by, e.g., Gershgorin circle theorem), we could replace $\bA^{(l)}$ with $\frac{1}{r}\bA^{(l)}$ to get $\max_{l}\|\bA^{(l)}\|_{2}\leq 1$ and arrive at the same eigen-space. Another way of addressing this generality is to adopt the idea of \cite{shamir2015fast}, that is, replacing the learning rate $\alpha$ with $r\alpha$ and the eigen-gap $\tau$ with $\tau/r$, with some of the constants in the theorem re-derived. 
The condition on the initial iterate, i.e., $k-\|\bV^{\top}\tilde{\bX}^{(0)}\|_{F}^{2}<\frac{1}{2}$, is theoretically non-trivial. However, empirically this condition can be well satisfied by running other stochastic algorithms (e.g., SRG-EIGS or DSRG-EIGS) or a few steps of deterministic iterative algorithms (e.g., RG-EIGS), because they are good at finding sub-optimal solutions.
In our experiments, we use SRG-EIGS for this purpose, which makes the theorem amount to a convergence analysis at a later stage of the hybrid algorithm (e.g., starting from $t_{0}>0$ instead of $t_{0}=0$). The convergence rate of our algorithm can be roughly identified by the iteration number $O(mT)=O(m\lceil\frac{\log(1/\varepsilon)}{\log(2/\varphi)}\rceil)$ which establishes an exponential global convergence rate. Compared to the sub-linear rate $O(1/\varepsilon)$ of DSRG-EIGS by \cite{zhiqiang2016},
it achieves a significant improvement since the complexity of a single iteration in the two algorithms only differs by constants.
In summary, initialized by a low-precision eigensolver, our SVRG-EIGS algorithm would obtain a high-precision solution in a limited number of epochs (data passes), which is theoretically guaranteed by Theorem \ref{thm}.

We provide an elegant proof of Theorem \ref{thm} in Appendix, though it is a bit involved. For ease of exposition and understanding, we decompose this course into three steps in a way similar to \cite{shamir2015fast}, including the analysis on one iteration, one epoch and one run of the algorithm. Among them, the first step (i.e., one iteration analysis) lies at the core of the main proof, where the techniques we use are dramatically different from those in \cite{NIPS2013_5132,shamir2015fast} due to our new context of Rimannian manifolds, or more precisely, Stiefel manifolds. This inherently different context requires new techniques, which in turn yield an improved exponential global convergence and accordingly bring more improvements over the convergence of sub-linear rate \cite{zhiqiang2016}.

\section{Experiments}\label{sec.exp}
In this section, we empirically verify the exponential convergence rate of our SVRRG-EIGS algorithm and demonstrate its capability of finding solutions of high precision when combined with other algorithms of low precision. Specifically, we use SRG-EIGS to generate a low-precision solution for initializing SVRRG-EIGS, and do the comparison with both RG-EIGS and SRG-EIGS. Among various implementations of RG-EIGS with different choices of metric and retraction in (\ref{eqn.rg-related}), we choose the one with canonical metric and Cayley transformation based retraction \cite{DBLP:journals/mp/WenY13} since its code is publically available\footnote{\url{optman.blogs.rice.edu/}}. This version of RG-EIGS uses the non-monotone line search with the well-known Barzilai-Borwein step size, which significantly reduces the iteration number, and performs well in practice. Both RG-EIGS and SRG-EIGS are fed with the same random initial value of $\bX$, where each entry is sampled from the standard normal distribution $\mathcal{N}(0,1)$ and then all entries as a whole are orthogonalized. SRG-EIGS uses the decaying learning rate $\alpha_{t}=\frac{\eta}{t}$ where $\eta$ will be tuned.

We verify the properties of our algorithm on a real symmetric matrix, {\bf Schenk}\footnote{\url{www.cise.ufl.edu/research/sparse/matrices/}}, of $10,728\times 10,728$ size, with $85,000$ nonzero entries. We partition $\bA$ into column blocks with block size equal to $100$ so that we can write $\bA=\frac{1}{L}\sum_{l=1}^{L}\bA^{(l)}$ with $L=\lceil\frac{10728}{100}\rceil$ and each $\bA^{(l)}$ having only one column block of $\bA$ and all others zero. We set $k=3$. For SVRRG-EIGS, we are able to use a fixed learning rate $\alpha$ and adopt the heuristic $\alpha=\frac{\zeta}{\|\bA\|_{1}\sqrt{n}}$ ($\|\cdot\|_{1}$ represents the matrix $1$-norm), similar to that in \cite{shamir2015fast}. We set $\zeta=4.442$ and epoch length $m=\frac{1}{2}L$, i.e., each epoch takes $1.5$ passes over $\bA$ (including one pass for computing the full gradient). Accordingly, the epoch length of SRG-EIGS is set to $m=\frac{3}{2}L$. In addition, we set $\bB^{(t)}=\bI$.

The performance of different algorithms is evaluated using three quality measures: feasibility $\|\bX^{(t)^{\top}}\bX^{(t)})-\bI\|_{F}$, relative error function $E(\bX)\triangleq1-\frac{\frac{1}{2}\tr(\bX^{(t)^{\top}}\bA\bX^{(t)})}{\max_{\bX\in\mathrm{St}(n,k)}\frac{1}{2}\tr(\bX^{(t)^{\top}}\bA\bX^{(t)}))}$, and normalized potential function $\mathbf{\Theta}(\tilde{\bX}^{(t)})/k=1-\frac{\|\bV^{\top}\tilde{\bX}^{(t)}\|_{F}^{2}}{k}$. The ground truths in these measures, including both $\bV$ and $\max_{\bX\in\mathrm{St}(n,k)}\frac{1}{2}\tr(\bX_{t}^{\top}\bA\bX_{t})$ that is set to $(1/2)\sum_{i=1}^{k}\lambda_{i}$, are obtained using Matlab's EIGS function for benchmarking. For each measure, lower values indicate higher quality.

Given a solution $\bX^{(0)}$ of low precision\footnote{This low precision could be problem dependent.} at $E(\bX^{(0)})\leq 10^{-6}$,
our SVRRG-EIGS targets a double precision, that is, $E(\bX)\leq 10^{-12}$ or $\mathbf{\Theta}(\tilde{\bX}^{(t)})/k\leq 10^{-12}$. Each algorithm terminates when the precision requirement is met or the maximum number of epoches (set as $20$) is reached.

\begin{figure*}[!ht]
\centerline{
\subfigure[Feasibility]{\includegraphics[width =.355\textwidth]{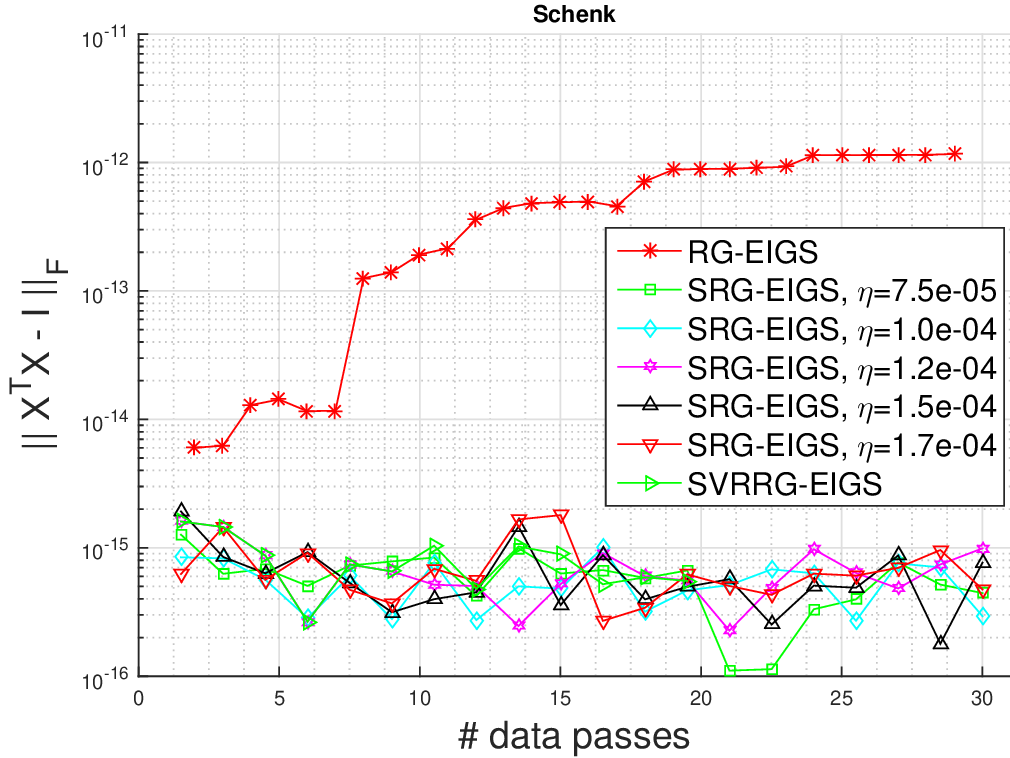}\label{fig:feasi}}\hspace{1.2mm}
\subfigure[Relative error function]{\includegraphics[width =.35\textwidth]{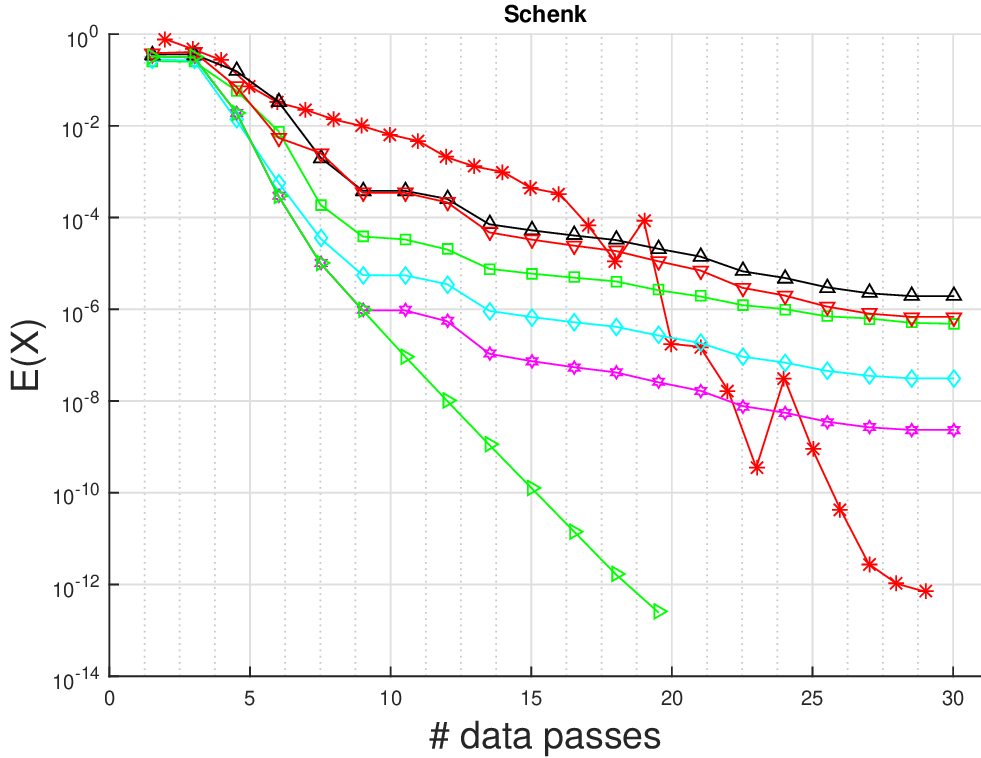}\label{fig:fval}}\hspace{1.2mm}
\subfigure[Normalized potential function]{\includegraphics[width =.35\textwidth]{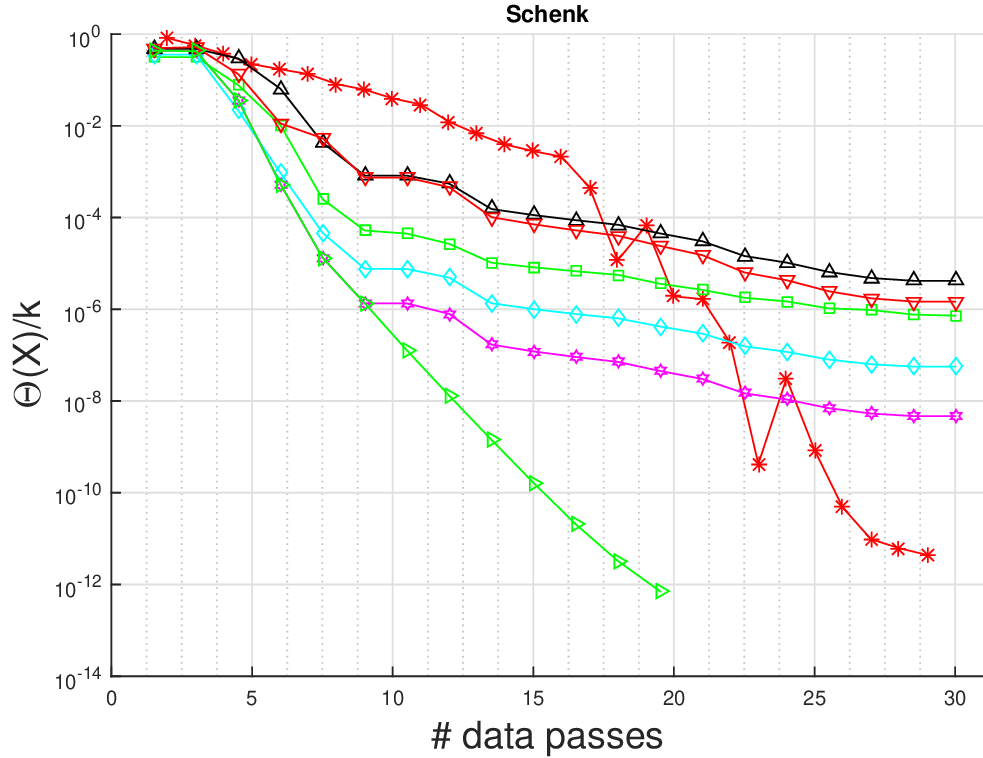}\label{fig:ktr}}
}
\caption{Performance on Schenk. Note that the y-axis in each figure is in log scale.}
\label{fig}
\vspace{-1mm}
\end{figure*}

We report the convergence curves in terms of each measure, on which empirical convergence rates of the algorithms can be observed. Figure \ref{fig} reports the performance of different algorithms. In terms of feasibility, both SRG-EIGS and SVRRG-EIGS perform well, while RG-EIGS produces much poorer results.
This is because the Cayley transformation based retraction used therein relies heavily on the Sherman-Morrison-Woodbury formula, which suffers from the numerical instability. From Figures \ref{fig:fval} and \ref{fig:ktr}, we observe similar convergence trends for each algorithm under the two different measures. All three algorithms improve their solutions with more iteration. There are several exceptions in RG-EIGS. This is due to the non-monotone step size used in its implementation. 
We also observe that SRG-EIGS presents an exponential convergence rate at an early stage thanks to a relatively large learning rate. However, it subsequently steps into a long period of sub-exponential convergence, which leads to small progress towards the optimal solution. In contrast, our SVRRG-EIGS inherits the initial momentum from SRG-EIGS and keeps the exponential convergence rate throughout the entire process. This enables it to approach the optimal solution at a fast speed. RG-EIGS has a different trend. It converges sub-exponentially at the beginning and performs the worst. Though it converges fast at a later stage, it still needs more passes over data than SVRRG-EIGS in order to achieve a high precision.

\section{Related Work} \label{related_work}
Existing methods on eigensolvers include the power method \cite{Golub:1996:MC:248979}, the (block) Lanczos algorithms \cite{Cullum:2002:LAL:640589}, Randomized SVD \cite{Halko:2011:FSR:2078879.2078881},
Riemannian methods \cite{1997IMA....92..135T,absil2008optimization}, and so on. All these methods performs the batch learning, while our focus in this paper is on stochastic algorithms. From this perspective, few existing works include online learning of eigenvectors \cite{DBLP:conf/icml/GarberHM15} which aims at the leading eigenvector, i.e., $k=1$, and doubly stochastic Riemannian method (DSRG-EIGS) \cite{zhiqiang2016}
where the learning rate has to decay to zero. \cite{DBLP:conf/icml/GarberHM15} provides the regret analysis without empirical verification for their method, while
DSRG-EIGS
belongs to one of implementations of SRG-EIGS in this paper where the double stochasticity comes from sampling over both data and coordinates of Riemmanian gradients. On the other hand, since the work of \cite{NIPS2013_4937}, variance reduction (SVRG) has become an appealing technique to stochastic optimization. There are quite some variants developed from different perspectives, such as practical SVRG \cite{NIPS2015_5711}, second-order SVRG \cite{kolte2015accelerating}, distributed or asynchronous SVRG \cite{DBLP:journals/corr/LeeML15,NIPS2015_5821}, and non-convex SVRG \cite{shamir2015fast,DBLP:journals/corr/ReddiHSPS16}. Our SVRRG belongs to non-convex SVRG, but is addressed from the Riemannian optimization perspective. The core techniques we use are dramatically different from existing ones due to our new context.

\section{Conclusion}\label{sec.con}
In this paper, we proposed the generalization of SVRG to Riemannian manifolds, and established the general framework of SVRG in this setting, SVRRG, which requires the key ingredient, vector transport, to make itself well-defined. It is then deployed to the eigensolver problem and induces the SVRRG-EIGS algorithm. We analyzed its theoretical properties in detail. As suggested by our theoretical results, the proposed algorithm is guaranteed to find high-precision solutions at an exponential convergence rate.
The theoretical implications are verified on a real dataset. For future work, we will explore the possibility of addressing the limitations of SVRRG-EIGS, e.g., dependence on eigen-gap and non-trivial initialization. We may also conduct more empirical investigations on the performance of SVRRG-EIGS.

\appendix
\section*{APPENDIX: Supplementary Material}

\section{Useful Lemmas}
In this section, some definition, basics, and a group of useful lemmas are provided. All the matrices are assumed to be real.

\subsection{Definitions and Basics}
\subsubsection{Matrix facts: symmetry, positive semi-definiteness, trace, norm and orthogonality}
The matrix $B\succeq 0$ ($\succ 0$) represents that $B$ is symmetric and positive semidefinite (definite), and if $B\succ 0$ then $B^{-1}\succ 0$ as well. The trace of a square matrix, $\tr(B)$, is the sum of diagonal entries of $B$. A useful fact about trace is the circular property, e.g., $\tr(BCD)=\tr(CDB)=\tr(DBC)$ for matrices $B,C,D$.  $\|B\|_{F}^{2}=\tr(B^{\top}B)=\tr(BB^{\top})=\|B^{\top}\|_{F}^{2}$ and $\|B\|_{2}=\sqrt{\lambda_{\max}(B^{\top}B)}=\sigma_{\max}(B)$ represents the Frobenious-norm and spectral norm (i.e., matrix $2$-norm) of matrix $B$, respectively. Here $\lambda_{\max}(\cdot)$ represents the maximum eigenvalue of an $n\times n$ matrix, $\sigma_{\max}(\cdot)$ represents the maximum singular value of an $n\times m$ matrix. Note that $BC$ and $CB$ have the same set of nonzero eigenvalues for two matrices $B\in\mathbb{R}^{n\times m}$ and $C\in\mathbb{R}^{m\times n}$. Thus, $\|B\|_{2}=\|B^{T}\|_{2}$. In this document, we always assume that the eigenvalues of an $n\times n$ matrix $B$ takes the form $\lambda_{1}(B)\geq \lambda_{2}(B)\geq \cdots \geq \lambda_{n}(B)$. Thus $\lambda_{\max}(B)=\lambda_{1}(B)$ and $|\lambda_{i}(B)|\leq\rho(B)\triangleq\max_{i}|\lambda_{i}|\leq\|B\|_{2}$ for any $i\in\{1,2,\cdots,n\}$, where $\rho(B)$ is called the spectral radius of a square matrix $B$. And also $\tr(B)=\sum_{i}\lambda_{i}(B)$, which in turn implies $\|C\|_{2}\leq \|C\|_{F}$ for any matrix $C$. For any two matrices $B$ and $C$ that make $BC$ well-defined, $\|BC\|\leq\|B\|\|C\|$ holds for both Frobenious-norm and spectral norm, and $\|BC\|_{F}\leq\|B\|_{F}\|C\|_{2}$ holds. Furthermore, the orthogonal invariance also holds for both Frobenious-norm and spectral norm, i.e., $\|PCQ^{\top}\|=\|C\|$ for column-orthonormal matrices $P$ and $Q$ (i.e., $P^{\top}P=I$ and $Q^{\top}Q=I$).  For $X\in\mathrm{St}(n,k)$, let $X_{\perp}$ represent its orthogonal complement, i.e., $[X, X_{\perp}][X, X_{\perp}]^{\top}=I$ which implies $X^{\top}X_{\perp}=0$ and $X_{\perp}\in\mathrm{St}(n,n-k)$.
\subsubsection{Martingale}
The filtration, defined on a measurable probability space, is an increasing sequence of sub-sigma algebras $\{\mathcal{F}_{t}\}$ for $t\geq 0$, meaning that $\mathcal{F}_{s}\subset \mathcal{F}_{t}$ for all $s\leq t$. In our context, $\mathcal{F}_{t}$ encodes the set of all the random variables seen thus far (i.e., from $0$ to $t$).  In this document, conditioned on $X^{(t)}$ refers to conditioned on $\mathcal{F}_{t}$ for brevity. Let $H=\{H_{t}\}$ and $\mathcal{F}=\{\mathcal{F}_{t}\}$ be a stochastic process and a filtration, respectively, on the same probability space. Then $H$ is called a martingale (super-martingale) with respect to $\mathcal{F}$ if for each $t$, $H_{t}$ is $\mathcal{F}_{t}$-measurable, $\mathbb{E}[|H_{t}|]<\infty$, and $\mathbb{E}[H_{t+1}|\mathcal{F}_{t}]=H_{t}$ ($\mathbb{E}[H_{t+1}|\mathcal{F}_{t}]\leq H_{t}$). Given a random variable $X\geq 0$ and a constant $a>0$, the probability $P(X\geq a)\leq \mathbb{E}[X]/a$ (Markov inequality). Let $X_{0},X_{1},\cdots,X_{T}$ be a martingale or supermartingale such that $|X_{t}-X_{t-1}|\leq d_{t}$ (i.e., bounded difference) where $d_{t}$ is a deterministic function of $t$. Then for all $t\geq 0$ and any $a>0$, the probability $P(X_{t}-X_{0}\geq a)\leq \exp\{-a^{2}/(2\sum_{s=1}^{t}d_{t}^{2})\}$ (Azuma-Hoeffding inequality) \cite{Mitzenmacher:2005:PCR:1076315}.

\subsection{Lemmas}
The proofs of Lemma \ref{lemma1}-\ref{lemma5} can be found in \cite{shamir2015fast}.

\begin{lemma} \label{lemma1}
For any $B,C,D\succeq 0$, it holds that
$$\tr(BC)\geq\tr(B(C-D)) \quad\textrm{and}\quad \tr(BC)\geq\tr((B-D)C).$$
\end{lemma}

\begin{lemma} \label{lemma2}
If $B\succeq 0$ and $C\succ 0$, then $\tr(BC^{-1})\geq\tr(B(2I-C))$.
\end{lemma}

\begin{lemma} \label{lemma3}
Let $B_{1},B_{2},Z_{1},Z_{2}$ be $k\times k$ square matrix, where $B_{1},B_{2}$ are fixed and $Z_{1},Z_{2}$ are stochastic zero-mean. Furthermore, suppose that for some fixed $\beta,\gamma,\delta>0$, it holds with probability $1$ that
\begin{itemize}
  \item For all $\nu\in [0,1]$, $B_{2}+\nu Z_{2}\succeq\delta I$
  \item $\max\{\|Z_{1}\|_{F},\|Z_{2}\|_{F}\}\leq\beta$
  \item $\|B_{1}+\alpha Z_{1}\|_{2}\leq\gamma$
\end{itemize}
Then
$$
\mathbb{E}[\tr((B_{1}+Z_{1})(B_{2}+Z_{2})^{-1})]\geq\tr(B_{1}B_{2}^{-1})-\frac{\beta^{2}(1+\gamma/\delta)}{\delta^{2}}.
$$
\end{lemma}

\begin{lemma} \label{lemma4}
Let $B$ be a $k\times k$ matrix with minimal singular value $\sigma$ and $\|B\|_{2}\leq 1$. Then
$$
1-\frac{\|B^{\top}B\|_{F}^{2}}{\|B\|_{F}^{2}}\geq \frac{\sigma^{2}}{k}(k-\|B\|_{F}^{2}).
$$
\end{lemma}

\begin{lemma} \label{lemma5}
For any $n\times k$ matrices $C,D$ with orthonormal columns, let $B^{\star}=\mathrm{arg}\min_{B^{\top}B=I}\|C-DB\|_{F}^{2}$. Then
$$
B^{\star}=Q_{2}Q_{1}^{\top},\quad \|C-DB^{\star}\|_{F}^{2}\leq \|C-DB\|_{F}^{2}\quad\textrm{and}\quad
\|C-DB^{\star}\|_{F}^{2} \leq 2(k-\|C^{\top}D\|_{F}^{2}),
$$
where $C^{\top}D=Q_{1}\Omega Q_{2}^{\top}$ is the SVD of $C^{\top}D$.
\end{lemma}

\begin{lemma} \label{lemma6}
Let $Y^{(t)}$ and $X^{(t+1)}$ be as defined in Section 3.2 of the main paper. Assume $\max_{l}\|A^{(l)}\|_{2}\leq 1$ and $\alpha<1/5$. Then for any $n\times k$ matrix $\bV$ with orthonormal columns, it holds that
$$
\Big|\|V^{\top}X^{(t+1)}\|_{F}^{2}-\|V^{\top}X^{(t)}\|_{F}^{2}\Big|\leq \frac{20k\alpha}{1-5\alpha}.
$$
\end{lemma}
\begin{proof}
Note that $I + \alpha_{t}^{2}\tilde{G}^{\top}(y_{t+1},X^{(t)})\tilde{G}(y_{t+1},X^{(t)})=Y^{(t+1)^{\top}}Y^{(t+1)}$ since $\tilde{G}(y_{t+1},X^{(t)})\in T_{\bX^{(t)}}\mathrm{St}(n,k)$ and thus $X^{(t)^{\top}}\tilde{G}(y_{t+1},X^{(t)})=0$ \cite{absil2008optimization}. Based on the proof of Lemma 9 in \cite{shamir2015fast}, it suffices for us to show that $Y^{(t+1)}=X^{(t)}+\alpha N$ and $\|N\|_{2}\leq 5$. In fact, from Section 3.2 of the main paper, we have
\begin{eqnarray*}
N&=&\tilde{G}(y_{t+1},X^{(t)})\\
&=&(I-X^{(t)}X^{(t)^{\top}})A_{t+1}X^{(t)} - \\
&&(I-X^{(t)}X^{(t)^{\top}})(I-\tilde{X}\tilde{X}^{\top})(A_{t+1}-A)\tilde{X} - X^{(t)}\mathrm{skew}(X^{(t)^{\top}}(I-\tilde{X}\tilde{X}^{\top})(A_{t+1}-A)\tilde{X})\\
&=& X_{\perp}^{(t)}X_{\perp}^{(t)^{\top}}A_{t+1}X^{(t)} - X_{\perp}^{(t)}X_{\perp}^{(t)^{\top}}\tilde{X}_{\perp}\tilde{X}_{\perp}^{\top}(A_{t+1}-A)\tilde{X} - X^{(t)}\mathrm{skew}(X^{(t)^{\top}}\tilde{X}_{\perp}\tilde{X}_{\perp}^{\top}(A_{t+1}-A)\tilde{X}).
\end{eqnarray*}
Since $\max_{l}\|A^{(l)}\|_{2}\leq 1$, we have $\|A_{t+1}\|_{2}\leq 1$, $\|A\|_{2}\leq 1$ and thus $\|A_{t+1}-A\|_{2}\leq \|A_{t+1}\|_{2} + \|A\|_{2}\leq 2$. Note that $\|X^{(t)}\|_{2}=\|X_{\perp}^{(t)}\|_{2}=\|\tilde{X}\|_{2}=\|\tilde{X}_{\perp}\|_{2}=1$. Then we have
\begin{eqnarray*}
\|X_{\perp}^{(t)}X_{\perp}^{(t)^{\top}}A_{t+1}X^{(t)}\|_{2}\leq \|X_{\perp}^{(t)}\|_{2}\|X_{\perp}^{(t)^{\top}}\|_{2}\|A_{t+1}\|_{2}\|X^{(t)}\|_{2}\leq 1.
\end{eqnarray*}
Similarly,
\begin{eqnarray*}
\|X_{\perp}^{(t)}X_{\perp}^{(t)^{\top}}\tilde{X}_{\perp}\tilde{X}_{\perp}^{\top}(A_{t+1}-A)\tilde{X} \|_{2}\leq 2,\\
\|X^{(t)}\mathrm{skew}(X^{(t)^{\top}}\tilde{X}_{\perp}\tilde{X}_{\perp}^{\top}(A_{t+1}-A)\tilde{X})\|_{2}\leq 2.
\end{eqnarray*}
Thus, $\|N\|_{2}\leq 5$.
\end{proof}

\begin{lemma} \label{lemma7}
If $B\succeq 0$ and $D\succeq C\succ 0$, then $\tr(BC^{-1})\geq\tr(BD^{-1})$.
\end{lemma}
\begin{proof}
\begin{eqnarray*}
BC^{-1}&=&B(D-(D-C))^{-1}\\
&=&BD^{-1/2}(I-D^{-1/2}(D-C)D^{-1/2})^{-1}D^{-1/2}.
\end{eqnarray*}
By Lemma \ref{lemma1}-\ref{lemma2}, we have
\begin{eqnarray*}
\tr(BC^{-1})&=&\tr(BD^{-1/2}(I-D^{-1/2}(D-C)D^{-1/2})^{-1}D^{-1/2})\\
&=&\tr(D^{-1/2}BD^{-1/2}(I-D^{-1/2}(D-C)D^{-1/2})^{-1})\\
&\geq & \tr(D^{-1/2}BD^{-1/2}(I+D^{-1/2}(D-C)D^{-1/2}))\\
&\geq & \tr(D^{-1/2}BD^{-1/2}) = \tr(BD^{-1})
\end{eqnarray*}
\end{proof}

\begin{lemma}[von Neumann's trace inequality \cite{Lewis96convexanalysis}] \label{lemma8}
For two symmetric $n\times n$ matrices $B$ and $C$, 
it holds that
$$\tr(BC)\leq\sum_{i=1}^{n}\lambda_{i}(B)\lambda_{i}(C).$$
\end{lemma}

\begin{lemma} \label{lemma9}
For two symmetric $n\times n$ matrices $B$ and $C$, it holds that
$$\tr(BC)\geq\max\{\sum_{i=1}^{n}\lambda_{n-i+1}(B)\lambda_{1}(C),\sum_{i=1}^{n}\lambda_{i}(B)\lambda_{n-i+1}(C)\}.$$
\end{lemma}
\begin{proof}
The proof is done by replacing $B$ with $-B$ or replacing $C$ with $-C$ in von Neumann's trace inequality.
\end{proof}

\section{Main Proof}
The proof of the theorem is a bit involved. For ease of exposition and understanding, we decompose this course into three steps in a way similar to \cite{shamir2015fast}, including the analysis on one iteration, one epoch and one run of the algorithm. Among them, the first step (i.e., one iteration analysis) lies at the core of the main proof, where the techniques we use are dramatically different from those in \cite{NIPS2013_5132,shamir2015fast} due to our new context of Rimannian manifolds, more precisely, Stiefel manifolds. This inherently different context requires new techniques,
which yield an improved exponential global convergence and accordingly bring more improvements over the convergence of sub-linear rate by \cite{zhiqiang2016}.

\subsection{One Iteration Analysis}
In the first step, we consider a single iteration $t$ of our SVRRG-EIGS algorithm. The goal here is to establish a stochastic recurrence relation on $\|V^{\top}X^{(t)}\|_{F}^{2}$ such that $\|V^{\top}X^{(t)}\|_{F}^{2}$ tends to $k$ as $t$ goes to infinity with high probability (w.h.p.). Note that $\|V^{\top}X^{(t)}\|_{F}^{2}\xrightarrow[\mathrm{w.h.p.}]{t\rightarrow\infty} k$
implies that $X^{(t)}$ converges to the global solution $V$ up to a $k\times k$ orthogonal matrix w.h.p., which is exactly one of our ultimate goals (i.e., convergence to global solutions w.h.p., fixed learning rate and exponential convergence rate). For brevity, we omit the lengthy superscripts by letting $X=X^{(t)}$, $X'=X^{(t+1)}$, $B=B^{(t)}$, and $\tilde{X} = \tilde{X}^{(s)}$. And assume that $\max_{l}\|A^{(l)}\|_{2}\leq 1$.

\begin{lemma} \label{lemma10}
Follow the notations and assumptions made in Lemma \ref{lemma13}. Then it holds that
\begin{eqnarray*}
\tr(X^{\top}VV^{\top}X_{\perp}X_{\perp}^{\top}AX) \geq \tau (\|V^{\top}X\|_{F}^{2}-\|X^{\top}VV^{\top}X\|_{F}^{2})
\end{eqnarray*}
\end{lemma}
\begin{proof}
Based on Section 2.1 of the main paper, the eigen-decomposition of matrix $A$ can be written as $A=V\Sigma V^{\top} + V_{\perp}\Sigma_{\perp} V_{\perp}^{\top}$. Then
\begin{eqnarray*}
&&\tr(X^{\top}VV^{\top}X_{\perp}X_{\perp}^{\top}AX) \\
&=& \tr(X^{\top}VV^{\top}X_{\perp}X_{\perp}^{\top}V\Sigma V^{\top}X) + \tr(X^{\top}VV^{\top}X_{\perp}X_{\perp}^{\top}V_{\perp}\Sigma_{\perp} V_{\perp}^{\top}X)\\
&=& \tr(V^{\top}XX^{\top}VV^{\top}X_{\perp}X_{\perp}^{\top}V\Sigma) + \tr(V_{\perp}^{\top}XX^{\top}VV^{\top}X_{\perp}X_{\perp}^{\top}V_{\perp}\Sigma_{\perp}).
\end{eqnarray*}
By Lemma \ref{lemma9}, we have
\begin{eqnarray*}
\tr(V^{\top}XX^{\top}VV^{\top}X_{\perp}X_{\perp}^{\top}V\Sigma)
&\geq& \sum_{i=1}^{k}\lambda_{i}(V^{\top}XX^{\top}VV^{\top}X_{\perp}X_{\perp}^{\top}V)\lambda_{k-i+1}(\Sigma),\; \textrm{and}\\
\tr(V_{\perp}^{\top}XX^{\top}VV^{\top}X_{\perp}X_{\perp}^{\top}V_{\perp}\Sigma_{\perp})
&\geq& \sum_{i=1}^{n-k}\lambda_{i}(V_{\perp}^{\top}XX^{\top}VV^{\top}X_{\perp}X_{\perp}^{\top}V_{\perp})\lambda_{n-k-i+1}(\Sigma_{\perp}).
\end{eqnarray*}
Note that both matrices above, i.e., $V^{\top}XX^{\top}VV^{\top}X_{\perp}X_{\perp}^{\top}V$ and $V_{\perp}^{\top}XX^{\top}VV^{\top}X_{\perp}X_{\perp}^{\top}V_{\perp}$ are symmetric and thus Lemma \ref{lemma9} can be applied. In fact,
\begin{eqnarray*}
V^{\top}XX^{\top}VV^{\top}X_{\perp}X_{\perp}^{\top}V
&=& V^{\top}XX^{\top}VV^{\top}(I-XX^{\top})V\\
&=& V^{\top}XX^{\top}V - V^{\top}XX^{\top}VV^{\top}XX^{\top}V\\
&=& V^{\top}XX^{\top}V - (V^{\top}XX^{\top}V)^{2},
\end{eqnarray*}
which is symmetric. Furthermore, it is positive seme-definite, because
$$
\rho(V^{\top}XX^{\top}V)\leq \|V^{\top}XX^{\top}V\|_{2}\leq (\|X^{\top}\|_{2}\|V\|_{2})^{2}=1,
$$
and thus
$$
\lambda_{i}(V^{\top}XX^{\top}VV^{\top}X_{\perp}X_{\perp}^{\top}V)= \lambda_{i}(V^{\top}XX^{\top}V) - \lambda_{i}^{2}(V^{\top}XX^{\top}V)\geq 0.
$$
Likewise, we have
\begin{eqnarray*}
V_{\perp}^{\top}XX^{\top}VV^{\top}X_{\perp}X_{\perp}^{\top}V_{\perp}
&=& V_{\perp}^{\top}(I-X_{\perp}X_{\perp}^{\top})VV^{\top}X_{\perp}X_{\perp}^{\top}V_{\perp}\\
&=& -V_{\perp}^{\top}X_{\perp}X_{\perp}^{\top}(I-V_{\perp}V_{\perp}^{\top})X_{\perp}X_{\perp}^{\top}V_{\perp}\\
&=& - (V_{\perp}^{\top}X_{\perp}X_{\perp}^{\top}V_{\perp}-(V_{\perp}^{\top}X_{\perp}X_{\perp}^{\top}V_{\perp})^{2}),
\end{eqnarray*}
which is symmetric but negative semi-definite, because
$$
\rho(V_{\perp}^{\top}X_{\perp}X_{\perp}^{\top}V_{\perp})\leq \|V_{\perp}^{\top}X_{\perp}X_{\perp}^{\top}V_{\perp}\|_{2}\leq(\|X_{\perp}^{\top}\|_{2}\|V_{\perp}\|_{2})^{2}=1,
$$
and thus
$$
\lambda_{i}(V_{\perp}^{\top}XX^{\top}VV^{\top}X_{\perp}X_{\perp}^{\top}V_{\perp})= -(\lambda_{i}(V_{\perp}^{\top}X_{\perp}X_{\perp}^{\top}V_{\perp})-\lambda_{i}^{2}(V_{\perp}^{\top}X_{\perp}X_{\perp}^{\top}V_{\perp}))\leq 0.
$$
We now can write
\begin{eqnarray*}
&&\tr(V^{\top}XX^{\top}VV^{\top}X_{\perp}X_{\perp}^{\top}V\Sigma)+\tr(V_{\perp}^{\top}XX^{\top}VV^{\top}X_{\perp}X_{\perp}^{\top}V_{\perp}\Sigma_{\perp})\\
&\geq& \sum_{i=1}^{k}\lambda_{i}(V^{\top}XX^{\top}VV^{\top}X_{\perp}X_{\perp}^{\top}V)\lambda_{k-i+1}(\Sigma)+ \sum_{i=1}^{n-k}\lambda_{i}(V_{\perp}^{\top}XX^{\top}VV^{\top}X_{\perp}X_{\perp}^{\top}V_{\perp})\lambda_{n-k-i+1}(\Sigma_{\perp})\\
&\geq & \sum_{i=1}^{k}\lambda_{i}(V^{\top}XX^{\top}VV^{\top}X_{\perp}X_{\perp}^{\top}V)\lambda_{k}(A)+ \sum_{i=1}^{n-k}\lambda_{i}(V_{\perp}^{\top}XX^{\top}VV^{\top}X_{\perp}X_{\perp}^{\top}V_{\perp})\lambda_{k+1}(A),
\end{eqnarray*}
in which, we find that
\begin{eqnarray*}
\sum_{i=1}^{k}\lambda_{i}(V^{\top}XX^{\top}VV^{\top}X_{\perp}X_{\perp}^{\top}V)
&=& \sum_{i=1}^{k}\lambda_{i}(V^{\top}XX^{\top}V) - \sum_{i=1}^{k}\lambda_{i}^{2}(V^{\top}XX^{\top}V)\\
&=& \tr(V^{\top}XX^{\top}V) - \tr((V^{\top}XX^{\top}V)^{2})\\
&=& \|X^{\top}V\|_{F}^{2} - \|V^{\top}XX^{\top}V\|_{F}^{2}
\end{eqnarray*}
and similarly
\begin{eqnarray*}
\sum_{i=1}^{k}\lambda_{i}(V_{\perp}^{\top}XX^{\top}VV^{\top}X_{\perp}X_{\perp}^{\top}V_{\perp})
&=& -(\|X_{\perp}^{\top}V_{\perp}\|_{F}^{2} - \|V_{\perp}^{\top}X_{\perp}X_{\perp}^{\top}V_{\perp}\|_{F}^{2}).
\end{eqnarray*}
Note that
\begin{eqnarray*}
\|X_{\perp}^{\top}V_{\perp}\|_{F}^{2} & = & \tr(V_{\perp}^{\top}X_{\perp}X_{\perp}^{\top}V_{\perp})=\tr(V_{\perp}V_{\perp}^{\top}X_{\perp}X_{\perp}^{\top})= \tr((I-VV^{\top})(I-XX^{\top}))\\
&=&\tr(I-VV^{\top}-XX^{\top}+VV^{\top}XX^{\top})\\
&=& n-2k + \|X^{\top}V\|_{F}^{2}
\end{eqnarray*}
and
\begin{eqnarray*}
&&\|V_{\perp}^{\top}X_{\perp}X_{\perp}^{\top}V_{\perp}\|_{F}^{2} \\
& = &
\tr((I-VV^{\top})(I-XX^{\top})(I-VV^{\top})(I-XX^{\top}))\\
&=&\tr((I-VV^{\top})(I-XX^{\top})(I-VV^{\top}-XX^{\top}+VV^{\top}XX^{\top}))\\
&=&\tr((I-VV^{\top})(I-XX^{\top})(I+VV^{\top}XX^{\top}))\\
&=&\tr(I-VV^{\top}-XX^{\top}+VV^{\top}XX^{\top}+ (I-VV^{\top}-XX^{\top}+VV^{\top}XX^{\top})VV^{\top}XX^{\top})\\
&=&\tr(I-VV^{\top}-XX^{\top}+VV^{\top}XX^{\top}VV^{\top}XX^{\top})\\
&=& n-2k + \|V^{\top}XX^{\top}V\|_{F}^{2}.
\end{eqnarray*}
Therefore, we arrive at
\begin{eqnarray*}
&&\tr(V^{\top}XX^{\top}VV^{\top}X_{\perp}X_{\perp}^{\top}V\Sigma)+\tr(V_{\perp}^{\top}XX^{\top}VV^{\top}X_{\perp}X_{\perp}^{\top}V_{\perp}\Sigma_{\perp})\\
&\geq & \lambda_{k}\sum_{i=1}^{k}\lambda_{i}(V^{\top}XX^{\top}VV^{\top}X_{\perp}X_{\perp}^{\top}V)+ \lambda_{k+1}\sum_{i=1}^{n-k}\lambda_{i}(V_{\perp}^{\top}XX^{\top}VV^{\top}X_{\perp}X_{\perp}^{\top}V_{\perp})\\
&=& (\lambda_{k}-\lambda_{k+1})(\|X^{\top}V\|_{F}^{2} - \|V^{\top}XX^{\top}V\|_{F}^{2})\\
&=& \tau (\|V^{\top}X\|_{F}^{2} - \|X^{\top}VV^{\top}X\|_{F}^{2}).
\end{eqnarray*}
\end{proof}

\begin{lemma} \label{lemma11}
Let $B_{1}$ and $B_{2}$ be defined by (\ref{eqn.B1_def}) and (\ref{eqn.B2_def}), respectively, and follow the notations and assumptions made in Lemma \ref{lemma13}. Then it holds that
\begin{eqnarray*}
\tr(B_{1}B_{2}^{-1})&\geq & \|V^{\top}X\|_{F}^{2} + 2\alpha\tau (\|V^{\top}X\|_{F}^{2}-\|X^{\top}VV^{\top}X\|_{F}^{2}) - \\
&& \alpha^{2}(1+2\alpha)(4(k - \|V^{\top}X\|_{F}^{2}) + k^{2}\kappa_{F}^{2}).
\end{eqnarray*}
\end{lemma}
\begin{proof}
Note that
\begin{eqnarray*}
B_{1} &=& (X^{\top}(I+\alpha AX_{\perp}X_{\perp}^{\top})VV^{\top}(I+\alpha X_{\perp}X_{\perp}^{\top}A)X)\succeq 0 \\
&=& X^{\top}VV^{\top}X + \alpha X^{\top}VV^{\top}X_{\perp}X_{\perp}^{\top}AX + \alpha X^{\top}AX_{\perp}X_{\perp}^{\top}VV^{\top}X +\\ &&\alpha^{2}X^{\top}AX_{\perp}X_{\perp}^{\top}VV^{\top}X_{\perp}X_{\perp}^{\top}AX
\end{eqnarray*}
and
\begin{eqnarray*}
B_{2} = I + \alpha^{2}X^{\top}AX_{\perp}X_{\perp}^{\top}AX + \alpha^{2}\kappa_{F}^{2}I \succ 0.
\end{eqnarray*}
Then by Lemma \ref{lemma2}, we get
\begin{eqnarray*}
\tr(B_{1}B_{2}^{-1}) &\geq & \tr(B_{1}(2I-B_{2}))\\
&=& \tr(B_{1}(I - \alpha^{2}X^{\top}AX_{\perp}X_{\perp}^{\top}AX - \alpha^{2}\kappa_{F}^{2}I)).
\end{eqnarray*}
Since $\alpha\leq\frac{1}{4}$, $\alpha\kappa_{F}\leq\frac{1}{4}$ and
$$
\|X^{\top}AX_{\perp}X_{\perp}^{\top}AX\|_{2}\leq (\|X_{\perp}^{\top}\|_{2}\|A\|_{2}\|X\|_{2})^{2}=\|A\|_{2}^{2}\leq 1,
$$
we have
\begin{eqnarray*}
&&I - \alpha^{2}X^{\top}AX_{\perp}X_{\perp}^{\top}AX - \alpha^{2}\kappa_{F}^{2}I\\
&\succeq & I - \alpha^{2}\|X^{\top}AX_{\perp}X_{\perp}^{\top}AX\|_{2}I- \alpha^{2}\kappa_{F}^{2}I\\
&\succeq & I - \alpha^{2} I - \alpha^{2}\kappa_{F}^{2}I \succeq (1-\frac{1}{16}-\frac{1}{16})I = \frac{7}{8}I\succ 0.
\end{eqnarray*}
And note that $\alpha^{2}X^{\top}AX_{\perp}X_{\perp}^{\top}VV^{\top}X_{\perp}X_{\perp}^{\top}AX\succeq 0$. Then by Lemma \ref{lemma1}, we can arrive at
\begin{eqnarray*}
&&\tr(B_{1}B_{2}^{-1})\\
&\geq & \tr((B_{1}-\alpha^{2}X^{\top}AX_{\perp}X_{\perp}^{\top}VV^{\top}X_{\perp}X_{\perp}^{\top}AX)(I - \alpha^{2}X^{\top}AX_{\perp}X_{\perp}^{\top}AX - \alpha^{2}\kappa_{F}^{2}I)).
\end{eqnarray*}
To simplify above inequality, define
\begin{eqnarray*}
C_{1}&\triangleq & B_{1}-\alpha^{2}X^{\top}AX_{\perp}X_{\perp}^{\top}VV^{\top}X_{\perp}X_{\perp}^{\top}AX\\
&=& X^{\top}VV^{\top}X + \alpha X^{\top}VV^{\top}X_{\perp}X_{\perp}^{\top}AX + \alpha X^{\top}AX_{\perp}X_{\perp}^{\top}VV^{\top}X.
\end{eqnarray*}
Then
\begin{eqnarray*}
\tr(B_{1}B_{2}^{-1})
&\geq & \tr(C_{1}(I - \alpha^{2}X^{\top}AX_{\perp}X_{\perp}^{\top}AX - \alpha^{2}\kappa_{F}^{2}I))\\
&=& \tr(C_{1})-\alpha^{2}\tr(C_{1}X^{\top}AX_{\perp}X_{\perp}^{\top}AX)-\alpha^{2}\kappa_{F}^{2}\tr(C_{1}).
\end{eqnarray*}
we now lower bound each of three items above. On one hand, by Lemma \ref{lemma10}, we get
\begin{eqnarray*}
\tr(C_{1})
&=& \tr(X^{\top}VV^{\top}X) + 2\alpha \tr(X^{\top}VV^{\top}X_{\perp}X_{\perp}^{\top}AX)\\
&\geq& \|V^{\top}X\|_{F}^{2} + 2\alpha \tau (\|V^{\top}X\|_{F}^{2} - \|X^{\top}VV^{\top}X\|_{F}^{2}).
\end{eqnarray*}
On the other hand, by Cauchy-Schwarz inequality, we can obtain
\begin{eqnarray*}
\tr(C_{1})
&=& \tr(X^{\top}VV^{\top}X) + 2\alpha \tr(X^{\top}VV^{\top}X_{\perp}X_{\perp}^{\top}AX)\\
&\leq& \|X^{\top}VV^{\top}\|_{F}\|X\|_{F} + 2\alpha \|X^{\top}VV^{\top}X_{\perp}X_{\perp}^{\top}A\|_{F}\|X\|_{F}\\
&\leq& \|X^{\top}\|_{F}\|VV^{\top}\|_{2}\|X\|_{F} + 2\alpha \|X^{\top}\|_{F}\|VV^{\top}X_{\perp}X_{\perp}^{\top}A\|_{2}\|X\|_{F}\\
&\leq& \|VV^{\top}\|_{2}\|X\|_{F}^{2} + 2\alpha \|VV^{\top}\|_{2}\|X_{\perp}X_{\perp}^{\top}\|_{2}\|A\|_{2}\|X\|_{F}^{2}\\
&\leq& (1+2\alpha)\|X\|_{F}^{2}=(1+2\alpha)k^{2}.
\end{eqnarray*}
For the middle term, noting that $\|\Sigma\|_{2}\leq \|A\|_{2}$ and $\|\Sigma_{\perp}\|_{2}\leq \|A\|_{2}$, then it can be derived as follows
\begin{eqnarray*}
&& \tr(C_{1}X^{\top}AX_{\perp}X_{\perp}^{\top}AX)\\
&\leq & \|C_{1}X^{\top}AX_{\perp}\|_{F}\|X_{\perp}^{\top}AX\|_{F}
\leq  \|C_{1}\|_{2}\|X^{\top}AX_{\perp}\|_{F}\|X_{\perp}^{\top}AX\|_{F}\\
&\leq & (\|X^{\top}VV^{\top}X\|_{2} + 2\alpha \|X^{\top}VV^{\top}X_{\perp}X_{\perp}^{\top}AX\|_{2})\|X^{\top}AX_{\perp}\|_{F}^{2}\\
&\leq & (1+2\alpha)\|X^{\top}AX_{\perp}\|_{F}^{2}\\
&=&(1+2\alpha)\|X^{\top}(V\Sigma V^{\top} + V_{\perp}\Sigma_{\perp} V_{\perp}^{\top})X_{\perp}\|_{F}^{2}\\
&\leq & (1+2\alpha)(\|X^{\top}V\Sigma V^{\top}X_{\perp}\|_{F} + \|X^{\top}V_{\perp}\Sigma_{\perp} V_{\perp}^{\top}X_{\perp}\|_{F})^{2}\\
&\leq & (1+2\alpha)(\|X^{\top}\|_{2}\|V\|_{2}\|\Sigma\|_{2}\|V^{\top}X_{\perp}\|_{F} + \|X^{\top}V_{\perp}\|_{F}\|\Sigma_{\perp}\|_{2} \|V_{\perp}^{\top}\|_{2}\|X_{\perp}\|_{2})^{2}\\
&\leq & (1+2\alpha)(\|A\|_{2}\|V^{\top}X_{\perp}\|_{F} + \|X^{\top}V_{\perp}\|_{F}\|A\|_{2})^{2}\\
&\leq & (1+2\alpha)(\|V^{\top}X_{\perp}\|_{F} + \|X^{\top}V_{\perp}\|_{F})^{2},
\end{eqnarray*}
where
\begin{eqnarray*}
\|V^{\top}X_{\perp}\|_{F}^{2}&=&\tr(X_{\perp}^{\top}VV^{\top}X_{\perp})=\tr(VV^{\top}X_{\perp}X_{\perp}^{\top})\\
&=&\tr(VV^{\top}(I-XX^{\top})) = k - \tr(VV^{\top}XX^{\top}) \\
&=& k - \tr(X^{\top}VV^{\top}X) = k - \|V^{\top}X\|_{F}^{2},
\end{eqnarray*}
and similarly $\|X^{\top}V_{\perp}\|_{F}^{2}=k - \|V^{\top}X\|_{F}^{2}$.
Thus, we could write
\begin{eqnarray*}
\tr(C_{1}X^{\top}AX_{\perp}X_{\perp}^{\top}AX)
&\leq & (1+2\alpha)(\|V^{\top}X_{\perp}\|_{F} + \|X^{\top}V_{\perp}\|_{F})^{2}\\
&= & 4(1+2\alpha)(k - \|V^{\top}X\|_{F}^{2}).
\end{eqnarray*}
Therefore, we now can arrive at
\begin{eqnarray*}
\tr(B_{1}B_{2}^{-1})
&\geq & \tr(C_{1})-\alpha^{2}\tr(C_{1}X^{\top}AX_{\perp}X_{\perp}^{\top}AX)-\alpha^{2}\kappa_{F}^{2}\tr(C_{1})\\
&\geq & \|V^{\top}X\|_{F}^{2} + 2\alpha \tau (\|X^{\top}V\|_{F}^{2} - \|V^{\top}XX^{\top}V\|_{F}^{2})-\\
&& \alpha^{2}(1+2\alpha)(4(k - \|V^{\top}X\|_{F}^{2}) + k^{2}\kappa_{F}^{2}).
\end{eqnarray*}

\end{proof}

\begin{lemma} \label{lemma12}
Follow the notations made in Lemma \ref{lemma13}, assume $A=\frac{1}{L}\sum_{l=1}^{L}A^{(l)}$ with $\max_{l}\|A^{(l)}\|_{2}\leq 1$ (thus $\|A\|_{2}\leq 1$),
and let
\begin{eqnarray*}
W &=& (I-XX^{\top})(A_{t+1}-A)(X-\tilde{X}B)+(I-XX^{\top})\tilde{X}\tilde{X}^{\top}(A_{t+1}-A)\tilde{X}B- \\
&&X\mathrm{skew}(X^{\top}(I-\tilde{X}\tilde{X}^{\top})(A_{t+1}-A)\tilde{X}B)
\end{eqnarray*}
recalling from Section 3.2 of the main paper. Then it holds that
$\mathbb{E}[W|X]=0$ and we can take
\begin{eqnarray*}
\kappa_{2} = 8\quad\textrm{and}\quad\kappa_{F}^{2}=96(k-\|V^{\top}X\|_{F}^{2}+k-\|V^{\top}\tilde{X}\|_{F}^{2}).
\end{eqnarray*}
\end{lemma}

\begin{proof}
Note that $\mathbb{E}[A_{t+1}]=A$ and $\mathbb{E}[B|X]=B$. Then we have
\begin{eqnarray*}
\mathbb{E}[W|X] &=&
(I-XX^{\top})(\mathbb{E}[A_{t+1}]-A)(X-\tilde{X}B)+(I-XX^{\top})\tilde{X}\tilde{X}^{\top}(\mathbb{E}[A_{t+1}]-A)\tilde{X}B- \\
&&X\mathrm{skew}(X^{\top}(I-\tilde{X}\tilde{X}^{\top})(\mathbb{E}[A_{t+1}]-A)\tilde{X}B)\\
&=& 0.
\end{eqnarray*}
We now upper bound the spectral norm and Frobenius norm of $W$. First we rewrite it as
\begin{eqnarray*}
W &=& X_{\perp}X_{\perp}^{\top}(A_{t+1}-A)(X-\tilde{X}B)+X_{\perp}X_{\perp}^{\top}\tilde{X}\tilde{X}^{\top}(A_{t+1}-A)\tilde{X}B- \\
&&X\mathrm{skew}(X^{\top}\tilde{X}_{\perp}\tilde{X}_{\perp}^{\top}(A_{t+1}-A)\tilde{X}B).
\end{eqnarray*}
Noting that $B^{\top}B=BB^{\top}=I$, we get
\begin{eqnarray*}
\|W\|_{2} &\leq & \|X_{\perp}X_{\perp}^{\top}(A_{t+1}-A)(X-\tilde{X}B)\|_{2}+\|X_{\perp}X_{\perp}^{\top}\tilde{X}\tilde{X}^{\top}(A_{t+1}-A)\tilde{X}B\|_{2}+ \\
&&\|X\mathrm{skew}(X^{\top}\tilde{X}_{\perp}\tilde{X}_{\perp}^{\top}(A_{t+1}-A)\tilde{X}B)\|_{2}\\
&\leq & (\|A_{t+1}\|_{2}+\|A\|_{2})(\|X\|_{2}+\|\tilde{X}B\|_{2})+2(\|A_{t+1}\|_{2}+\|A\|_{2})\\
&\leq & 8 \triangleq \kappa_{2},
\end{eqnarray*}
while
\begin{eqnarray*}
\|W\|_{F}^{2} &\leq & (\|X_{\perp}X_{\perp}^{\top}(A_{t+1}-A)(X-\tilde{X}B)\|_{F}+\|X_{\perp}X_{\perp}^{\top}\tilde{X}\tilde{X}^{\top}(A_{t+1}-A)\tilde{X}B\|_{F}+ \\
&&\|X\mathrm{skew}(X^{\top}\tilde{X}_{\perp}\tilde{X}_{\perp}^{\top}(A_{t+1}-A)\tilde{X}B)\|_{F})^{2}\\
&\leq & (\|X_{\perp}X_{\perp}^{\top}(A_{t+1}-A)\|_{2}\|X-\tilde{X}B\|_{F}+\|X_{\perp}\|_{2}\|X_{\perp}^{\top}\tilde{X}\|_{F}\|\tilde{X}^{\top}(A_{t+1}-A)\tilde{X}B\|_{2}+ \\
&&\|X\|_{2}\|X^{\top}\tilde{X}_{\perp}\|_{F}\|\tilde{X}_{\perp}^{\top}(A_{t+1}-A)\tilde{X}B\|_{2})^{2}\\
&\leq & 4(\|X-\tilde{X}B\|_{F}+\|X_{\perp}^{\top}\tilde{X}\|_{F}+\|X^{\top}\tilde{X}_{\perp}\|_{F})^{2}\\
&\leq & 12(\|X-\tilde{X}B\|_{F}^{2}+\|X_{\perp}^{\top}\tilde{X}\|_{F}^{2}+\|X^{\top}\tilde{X}_{\perp}\|_{F}^{2}).
\end{eqnarray*}
To proceed further, each of three items in above bracket needs to upper bounded. To this end, note that $B=\mathrm{arg}\min_{D}\|X-\tilde{X}D\|_{F}^{2}$ by the definition of $B$ and Lemma \ref{lemma5}. Then if we let $C^{\star}=\mathrm{arg}\min_{C}\|X - VC\|_{F}^{2}$ and $D^{\star}=\mathrm{arg}\min_{D}\|VC^{\star} - \tilde{X}D\|_{F}^{2}$, we can get
\begin{eqnarray*}
\|X-\tilde{X}B\|_{F}^{2} &\leq & \|X-\tilde{X}D^{\star}\|_{F}^{2}\quad\textrm{(Lemma \ref{lemma5})}\\
&= & \|X-VC^{\star} + VC^{\star}- \tilde{X}D^{\star}\|_{F}^{2}\\
&\leq & (\|X-VC^{\star}\|_{F} + \|VC^{\star}- \tilde{X}D^{\star}\|_{F})^{2}\\
&\leq & 2(\|X-VC^{\star}\|_{F}^{2} + \|VC^{\star}- \tilde{X}D^{\star}\|_{F}^{2})\\
&\leq & 4(k-\|X^{\top}V\|_{F}^{2} + k-\|C^{\star^{\top}}V^{\top}\tilde{X}\|_{F}^{2}) \quad\textrm{(Lemma \ref{lemma5})}\\
&= & 4(k-\|V^{\top}X\|_{F}^{2} + k-\|V^{\top}\tilde{X}\|_{F}^{2}). \quad\textrm{(orthogonal invariance)}
\end{eqnarray*}
For other two items, noting that $I=VV^{\top}+V_{\perp}V_{\perp}^{\top}$, we have
\begin{eqnarray*}
\|X_{\perp}^{\top}\tilde{X}\|_{F}^{2}&=& \|X_{\perp}^{\top}(VV^{\top}+V_{\perp}V_{\perp}^{\top})\tilde{X}\|_{F}^{2}\\
&\leq & (\|X_{\perp}^{\top}VV^{\top}\tilde{X}\|_{F}+\|X_{\perp}^{\top}V_{\perp}V_{\perp}^{\top}\tilde{X}\|_{F})^{2}\\
&\leq & (\|X_{\perp}^{\top}V\|_{F}\|V^{\top}\tilde{X}\|_{2}+\|X_{\perp}^{\top}V_{\perp}\|_{2}\|V_{\perp}^{\top}\tilde{X}\|_{F})^{2}\\
&\leq & (\|X_{\perp}^{\top}V\|_{F}\|V^{\top}\|_{2}\|\tilde{X}\|_{2}+\|X_{\perp}^{\top}\|_{2}\|V_{\perp}\|_{2}\|V_{\perp}^{\top}\tilde{X}\|_{F})^{2}\\
&= & (\|X_{\perp}^{\top}V\|_{F}+\|V_{\perp}^{\top}\tilde{X}\|_{F})^{2}\\
&\leq & 2(\|X_{\perp}^{\top}V\|_{F}^{2}+\|V_{\perp}^{\top}\tilde{X}\|_{F}^{2})\\
&\leq & 2(k-\|V^{\top}X\|_{F}^{2}+k-\|V^{\top}\tilde{X}\|_{F}^{2}),
\end{eqnarray*}
and similarly
\begin{eqnarray*}
\|X^{\top}\tilde{X}_{\perp}\|_{F}^{2}\leq 2(k-\|V^{\top}X\|_{F}^{2}+k-\|V^{\top}\tilde{X}\|_{F}^{2}).
\end{eqnarray*}
Therefore, we get
\begin{eqnarray*}
\|W\|_{F}^{2}\leq 96(k-\|V^{\top}X\|_{F}^{2}+k-\|V^{\top}\tilde{X}\|_{F}^{2})\triangleq \kappa_{F}^{2}.
\end{eqnarray*}

\end{proof}

\begin{lemma} \label{lemma13}
Assume $A$ is an $n\times n$ symmetric matrix with the eigenvalues $\lambda_{1}\geq\lambda_{2}\geq\cdots\geq\lambda_{n}$ and the eigen-gap $\tau=\lambda_{k}-\lambda_{k+1}>0$. And it could be written as $A=\frac{1}{L}\sum_{l=1}^{L}A^{(l)}$ with $\max_{l}\|A^{(l)}\|_{2}\leq 1$ (thus $\|A\|_{2}\leq 1$). Let $W$ be an $n\times k$ stochastic zero-mean matrix (i.e., $\mathbb{E}[W]=0$) with $\|W\|_{2}\leq\kappa_{2}$ and $\|W\|_{2}\leq\kappa_{F}$ almost surely. Let $X\in\mathrm{St}(n,k)$ and define
$$
Y = (I+\alpha (I-XX^{\top})A)X + \alpha W,\quad X' = Y(Y^{\top}Y)^{-1/2}
$$
for some $\alpha\in [0,\frac{1}{4\max\{0,\kappa_{F}\}}]$. If $V\in\mathrm{St}(n,k)$ consisting of $A$'s $k$ eigenvectors corresponding to eigenvalues $\lambda_{1}\geq\lambda_{2}\geq\cdots\geq\lambda_{k}$ and accordingly $V_{\perp}\in\mathrm{St}(n,n-k)$ consisting of $A$'s $n-k$ eigenvectors corresponding to eigenvalues $\lambda_{k+1}\geq\lambda_{k+2}\geq\cdots\geq\lambda_{n}$, 
then it holds that
\begin{eqnarray*}
\mathbb{E}[\|V^{\top}X'\|_{F}^{2}]&\geq &
\|V^{\top}X\|_{F}^{2} + 2\alpha \tau (\|V^{\top}X\|_{F}^{2} - \|X^{\top}VV^{\top}X\|_{F}^{2})-\\
&& \alpha^{2}(1+2\alpha)(4(k - \|V^{\top}X\|_{F}^{2}) + k^{2}\kappa_{F}^{2}) - \frac{200}{27}\alpha^{2}(31+10\kappa_{2})\kappa_{F}^{2}
\end{eqnarray*}
\end{lemma}
\begin{proof}
First, we have
\begin{eqnarray*}
  \|V^{\top}X'\|_{F}^{2} &=& \tr(X'^{\top}VV^{\top}X') \\
     &=& \tr((Y^{\top}Y)^{-1/2}Y^{\top}VV^{\top}Y(Y^{\top}Y)^{-1/2}) \\
     &=& \tr(Y^{\top}VV^{\top}Y(Y^{\top}Y)^{-1}).
\end{eqnarray*}
Using the definition of $Y$ and the fact $XX^{\top}+X_{\perp}X_{\perp}^{\top}=I$, we have the expansion $Y^{\top}VV^{\top}Y = \tilde{B}_{1} + Z_{1}\succeq 0$ where
\begin{eqnarray*}
\tilde{B}_{1} &=& X^{\top}(I+\alpha AX_{\perp}X_{\perp}^{\top})VV^{\top}(I+\alpha X_{\perp}X_{\perp}^{\top}A)X + \alpha^{2}W^{\top}VV^{\top}W\succ 0,\\
Z_{1} &=& \alpha X^{\top}(I+\alpha AX_{\perp}X_{\perp}^{\top})VV^{\top}W + \alpha W^{\top}VV^{\top}(I+\alpha X_{\perp}X_{\perp}^{\top}A)X.
\end{eqnarray*}
Similarly, $Y^{\top}Y$ can be written as $Y^{\top}Y = \tilde{B}_{2} + Z_{2}$ with
\begin{eqnarray*}
\tilde{B}_{2} &=& X^{\top}(I+\alpha AX_{\perp}X_{\perp}^{\top})(I+\alpha X_{\perp}X_{\perp}^{\top}A)X + \alpha^{2}W^{\top}W\\
&=& I + \alpha^{2}X^{\top}AX_{\perp}X_{\perp}^{\top}AX + \alpha^{2}W^{\top}W\succ 0,\\
Z_{2} &=& \alpha X^{\top}(I+\alpha AX_{\perp}X_{\perp}^{\top})W + \alpha W^{\top}(I+\alpha X_{\perp}X_{\perp}^{\top}A)X.
\end{eqnarray*}
Then we get
\begin{eqnarray*}
\|V^{\top}X'\|_{F}^{2} = \tr((\tilde{B}_{1} + Z_{1})(\tilde{B}_{2} + Z_{2})^{-1}).
\end{eqnarray*}
Note that $W^{\top}W\preceq \lambda_{\max}(W^{\top}W)I=\|W\|_{2}^{2}I\leq \|W\|_{F}^{2}I\leq \kappa_{F}^{2}I$. Thus,
\begin{eqnarray}
\tilde{B}_{2} \preceq  I + \alpha^{2}X^{\top}AX_{\perp}X_{\perp}^{\top}AX + \alpha^{2}\kappa_{F}^{2}I \;\; \triangleq B_{2} \label{eqn.B2_def}
\end{eqnarray}
In addition, let
\begin{eqnarray}
B_{1} \triangleq X^{\top}(I+\alpha AX_{\perp}X_{\perp}^{\top})VV^{\top}(I+\alpha X_{\perp}X_{\perp}^{\top}A)X. \label{eqn.B1_def}
\end{eqnarray}
and note that $\alpha^{2}W^{\top}VV^{\top}W\succeq 0$. Then by Lemma \ref{lemma7} and \ref{lemma1}, we obtain
\begin{eqnarray*}
\|V^{\top}X'\|_{F}^{2} &\geq & \tr((B_{1} + \alpha^{2}W^{\top}VV^{\top}W + Z_{1})(B_{2} + Z_{2})^{-1})\\
&\geq & \tr((B_{1} + Z_{1})(B_{2} + Z_{2})^{-1})
\end{eqnarray*}
We now would like to apply Lemma\footnote{Note that this lemma was mistakenly applied in \cite{shamir2015fast} since $B_{1}$ and $B_{2}$ are not fixed. It could be rectified by what we do here.} \ref{lemma3} for removing $Z_{1}$ and $Z_{2}$. Doing so needs to meet the conditions of Lemma \ref{lemma3}. In fact,
\begin{itemize}
  \item $Z_{1}$ and $Z_{2}$ are stochastic zero-mean: $Z_{1}$ and $Z_{2}$ are linear functions of the stochastic zero-mean matrix $W$. Thus they are stochastic zero-mean as well.
  \item $B_{1}$ and $B_{2}$ are fixed. This is true since no stochastic quantities are involved.
  \item $B_{2}+\nu Z_{2}\succeq\frac{3}{8} I$ for all $\nu\in [0,1]$. It's easy to see that $B_{2}\succeq 0$, and meanwhile since $\|A\|_{2}\leq 1$, $\|X\|_{2}=\|X_{\perp}\|_{2}=1$, $\alpha\kappa_{F}\leq\frac{1}{4}$ and $\alpha\leq\frac{1}{4}$, we get
  \begin{eqnarray*}
  \|Z_{2}\|_{2} &\leq & 2\alpha\|W^{\top}(I+\alpha X_{\perp}X_{\perp}^{\top}A)X\|_{2}\\
  &\leq & 2\alpha\|W\|_{2}(\|I\|_{2}+\alpha \|X_{\perp}\|_{2}\|X_{\perp}^{\top}\|_{2}\|A\|_{2})\|X\|_{2}\\
  &\leq & 2\alpha\|W\|_{F}(1+\alpha) \leq 2\alpha\kappa_{F}(1+\alpha)\leq\frac{1}{2}(1+\alpha)\leq\frac{5}{8}.
  \end{eqnarray*}
  Note that $Z_{2}$ is symmetric and $\rho(Z_{2})\leq \|Z_{2}\|_{2}$. We thus could write $Z_{2}\succeq -\rho(Z_{2})I\succeq -\|Z_{2}\|_{2}I$. Then $B_{2}+\nu Z_{2}\succeq (1 - v\|Z_{2}\|_{2})I\succeq (1 - \|Z_{2}\|_{2})I\succeq\frac{3}{8}I$.
  \item $\max\{\|Z_{1}\|_{F},\|Z_{2}\|_{F}\}\leq\frac{5}{2}\alpha\kappa_{F}$. Note that $\|V\|_{2}=1$. Then
  \begin{eqnarray*}
  \|Z_{1}\|_{F} &\leq & 2\alpha\|W^{\top}VV^{\top}(I+\alpha X_{\perp}X_{\perp}^{\top}A)X\|_{F}\\
  &\leq & 2\alpha\|W^{\top}\|_{F}\|V\|_{2}\|V^{\top}\|_{2}(\|I\|_{2}+\alpha \|X_{\perp}\|_{2}\|X_{\perp}^{\top}\|_{2}\|A\|_{2})\|X\|_{2}\\
  &\leq & 2\alpha(1+\alpha)\kappa_{F}\leq 2\alpha(1+\frac{1}{4})\kappa_{F}=\frac{5}{2}\alpha\kappa_{F}.
  \end{eqnarray*}
  Similarly, we could get $\|Z_{1}\|_{F}\leq \frac{5}{2}\alpha\kappa_{F}$.
  \item $\|B_{1}+\alpha Z_{1}\|_{2}\leq \frac{5(5+2\kappa_{2})}{16}$. Note that similar to right above, we could have $\|Z_{1}\|_{2}\leq \frac{5}{2}\alpha\kappa_{2}$. Then
  \begin{eqnarray*}
      \|B_{1}\|_{2} &\leq & \|V^{\top}(I+\alpha X_{\perp}X_{\perp}^{\top}A)X\|_{2}^{2}\\
      &\leq & (\|V^{\top}\|_{2}(\|I\|_{2}+\alpha \|X_{\perp}\|_{2}\|X_{\perp}^{\top}\|_{2}\|A\|_{2})\|X\|_{2})^{2}\leq (1+\alpha)^{2}.
  \end{eqnarray*}
  Thus,
  \begin{eqnarray*}
  \|B_{1}+\alpha Z_{1}\|_{2} &\leq & \|B_{1}\|_{2}+\alpha \|Z_{1}\|_{2}\leq (1+\alpha)^{2} + \frac{5}{2}\alpha\kappa_{2}\\
  &\leq & (1+\frac{1}{4})^{2} + \frac{5}{8}\kappa_{2}=\frac{5(5+2\kappa_{2})}{16}.
  \end{eqnarray*}
\end{itemize}

Thus, we have $\delta = \frac{3}{8}$, $\beta=\frac{5}{2}\alpha\kappa_{F}$ and $\gamma=\frac{5(5+2\kappa_{2})}{16}$. Then by Lemma \ref{lemma3}, we get
\begin{eqnarray*}
\mathbb{E}[\|V^{\top}X'\|_{F}^{2}]&\geq & \mathbb{E}[\tr((B_{1} + Z_{1})(B_{2} + Z_{2})^{-1})]\\
&\geq & \tr(B_{1}B_{2}^{-1})
- \frac{400}{9}\alpha^{2}\kappa_{F}^{2}(1+\frac{5(5+2\kappa_{2})}{6})\\
&= &\tr(B_{1}B_{2}^{-1}) -\frac{200}{27}\alpha^{2}(31+10\kappa_{2})\kappa_{F}^{2}.
\end{eqnarray*}
Moreover, by Lemma \ref{lemma11}, we arrive at
\begin{eqnarray*}
\mathbb{E}[\|V^{\top}X'\|_{F}^{2}]&\geq &
\|V^{\top}X\|_{F}^{2} + 2\alpha \tau (\|V^{\top}X\|_{F}^{2} - \|X^{\top}VV^{\top}X\|_{F}^{2})-\\
&& \alpha^{2}(1+2\alpha)(4(k - \|V^{\top}X\|_{F}^{2}) + k^{2}\kappa_{F}^{2}) -\frac{200}{27}\alpha^{2}(31+10\kappa_{2})\kappa_{F}^{2}.
\end{eqnarray*}

\end{proof}

\begin{lemma} \label{lemma14}
Let $A$ and $W$ be defined as by our SVRRG-EIGS algorithm in Section 3.2 of the main paper. Assume that $A$ has the eigen-decomposition as defined in Section 2.1 of the main paper, the eigen-gap $\tau=\lambda_{k}-\lambda_{k+1}>0$, and $\max_{l}\|A^{(l)}\|_{2}\leq 1$. Further suppose that $\alpha=\mu\tau\in (0,\frac{1}{32\sqrt{3k}}]$ and $\|V^{\top}X\|_{F}^{2}\geq k-\frac{1}{2}$. Then it holds that
\begin{eqnarray*}
\mathbb{E}[k-\|V^{\top}X'\|_{F}^{2}]
&\leq &
(k-\|V^{\top}X\|_{F}^{2})(1-c_{1}\mu\tau^{2})+c_{2}\mu^{2}\tau^{2}(k-\|V^{\top}\tilde{X}\|_{F}^{2}),
\end{eqnarray*}
where the expectation is taken with respect to the random $y_{t+1}$ for $X'=X^{(t+1)}$ conditioned on $X=X^{(t)}$, and in addition $c_{1} = 2(\frac{1}{8}-2\mu(1+2\mu\tau)(1+24k^{2})-\frac{118400}{3}\mu)>0$ for any $\mu\in(0,c_{0})$ with
$$
c_{0}=\min\{\frac{1}{32\sqrt{3k\tau^{2}}},\frac{1}{c_{1}\tau^{2}},\frac{-(118406+144k^{2})+\sqrt{(118406+144k^{2})^{2}+18\tau(1+24k^{2})}}{24\tau(1+24k^{2})}\}>0,
$$
and $c_{2} = 96(k^{2}(1+2\mu\tau)+823)$.
\end{lemma}

\begin{proof}
First by Lemma \ref{lemma12}, $W$ is conditionally stochastic zero-mean. And $4\max\{1,\kappa_{F}\}\leq 4\max\{1,\sqrt{96\times 2k}\}=32\sqrt{3k}$. Then $\alpha\leq\frac{1}{32\sqrt{3k}}\leq\frac{1}{4\max\{1,\kappa_{F}\}}$. Thus Lemma \ref{lemma13} can be applied, and we have
\begin{eqnarray*}
\mathbb{E}[\|V^{\top}X'\|_{F}^{2}]&\geq &
\|V^{\top}X\|_{F}^{2} + 2\alpha \tau (\|V^{\top}X\|_{F}^{2} - \|X^{\top}VV^{\top}X\|_{F}^{2})-\\
&& 4\alpha^{2}(1+2\alpha)(k - \|V^{\top}X\|_{F}^{2})-\alpha^{2}(k^{2}(1+2\alpha)+\frac{200}{27}(31+10\kappa_{2}))\kappa_{F}^{2}.
\end{eqnarray*}
Let $\sigma$ be the minimum singular value of $V^{\top}X$. Since $\|V^{\top}X\|_{2}\leq \|V^{\top}\|_{2}\|X\|_{2}=1$, then by Lemma \ref{lemma4} we have
\begin{eqnarray*}
\mathbb{E}[\|V^{\top}X'\|_{F}^{2}]
&\geq & \|V^{\top}X\|_{F}^{2} + 2\alpha \tau \|V^{\top}X\|_{F}^{2}(1 - \frac{\|X^{\top}VV^{\top}X\|_{F}^{2}}{\|V^{\top}X\|_{F}^{2}})-\\
&& 4\alpha^{2}(1+2\alpha)(k - \|V^{\top}X\|_{F}^{2})-\alpha^{2}(k^{2}(1+2\alpha)+\frac{200}{27}(31+10\kappa_{2}))\kappa_{F}^{2}\\
&\geq & \|V^{\top}X\|_{F}^{2} + 2\alpha \tau \frac{\delta^{2}}{k}\|V^{\top}X\|_{F}^{2}(1 - \|V^{\top}X\|_{F}^{2})-\\
&& 4\alpha^{2}(1+2\alpha)(k - \|V^{\top}X\|_{F}^{2})-\alpha^{2}(k^{2}(1+2\alpha)+\frac{200}{27}(31+10\kappa_{2}))\kappa_{F}^{2},
\end{eqnarray*}
and then by Lemma \ref{lemma12},
\begin{eqnarray*}
&&\mathbb{E}[k-\|V^{\top}X'\|_{F}^{2}]\\
&\leq & k-\|V^{\top}X\|_{F}^{2} - 2\alpha \tau \frac{\sigma^{2}}{k}\|V^{\top}X\|_{F}^{2}(1 - \|V^{\top}X\|_{F}^{2})+4\alpha^{2}(1+2\alpha)(k - \|V^{\top}X\|_{F}^{2})+\\
&& 96\alpha^{2}(k^{2}(1+2\alpha)+\frac{200}{27}(31+80))(k-\|V^{\top}X\|_{F}^{2}+k-\|V^{\top}\tilde{X}\|_{F}^{2})\\
&= & (k-\|V^{\top}X\|_{F}^{2})(1-(2\alpha \frac{\tau\sigma^{2}}{k}\|V^{\top}X\|_{F}^{2}-4\alpha^{2}(1+2\alpha)-96\alpha^{2}(k^{2}(1+2\alpha)+\frac{200}{27}(31+80)))) \\
&& +\; 96\alpha^{2}(k^{2}(1+2\alpha)+\frac{200}{27}(31+80))(k-\|V^{\top}\tilde{X}\|_{F}^{2})\\
&= & (k-\|V^{\top}X\|_{F}^{2})(1-2\alpha(\frac{\tau\sigma^{2}}{k}\|V^{\top}X\|_{F}^{2}-2\alpha(1+2\alpha)(1+24k^{2})-\frac{118400}{3}\alpha)) \\
&& +\; 96\alpha^{2}(k^{2}(1+2\alpha)+\frac{7400}{9})(k-\|V^{\top}\tilde{X}\|_{F}^{2}).
\end{eqnarray*}
Note that $\|V^{\top}X\|_{2}\leq 1$ implies the singular values of $V^{\top}X$ fall into $[0,1]$. If $\sigma<\frac{1}{2}$ then $\|V^{\top}X\|_{F}^{2}< k-1 + \frac{1}{2}=k-\frac{1}{2}$, which contradicts the assumption $\|V^{\top}X\|_{F}^{2}\geq k-\frac{1}{2}$. Thus, $\sigma\geq\frac{1}{2}$. Furthermore, $\|V^{\top}X\|_{F}^{2}\geq k-\frac{1}{2}=\frac{k}{2}+\frac{k}{2}-\frac{1}{2}\geq \frac{k}{2}$. We thus get
\begin{eqnarray*}
\mathbb{E}[k-\|V^{\top}X'\|_{F}^{2}]
&\leq &
(k-\|V^{\top}X\|_{F}^{2})(1-2\alpha(\frac{1}{8}\tau-2\alpha(1+2\alpha)(1+24k^{2})-\frac{118400}{3}\alpha)) \\
&& +\; 96\alpha^{2}(k^{2}(1+2\alpha)+\frac{7400}{9})(k-\|V^{\top}\tilde{X}\|_{F}^{2}).
\end{eqnarray*}
Since $\alpha=\mu\tau$, $0<\mu\leq\frac{1}{32\sqrt{3k\tau^{2}}}$. Then
\begin{eqnarray*}
&&2\alpha(\frac{1}{8}\tau-2\alpha(1+2\alpha)(1+24k^{2})-\frac{118400}{3}\alpha)\\
&=& 2\mu\tau^{2}(\frac{1}{8}-2\mu(1+2\mu\tau)(1+24k^{2})-\frac{118400}{3}\mu),
\end{eqnarray*}
and
\begin{eqnarray*}
96\alpha^{2}(k^{2}(1+2\alpha)+\frac{7400}{9}) &=& 96\mu^{2}\tau^{2}(k^{2}(1+2\mu\tau)+\frac{7400}{9}).
\end{eqnarray*}
Further let
\begin{eqnarray*}
c_{1} &=& 2(\frac{1}{8}-2\mu(1+2\mu\tau)(1+24k^{2})-\frac{118400}{3}\mu),\\
c_{2} &=& 96(k^{2}(1+2\mu\tau)+\frac{7400}{9}).
\end{eqnarray*}
We then arrive at
\begin{eqnarray*}
\mathbb{E}[k-\|V^{\top}X'\|_{F}^{2}]
&\leq &
(k-\|V^{\top}X\|_{F}^{2})(1-c_{1}\mu\tau^{2})+c_{2}\mu^{2}\tau^{2}(k-\|V^{\top}\tilde{X}\|_{F}^{2}).
\end{eqnarray*}
And solving the equation $c_{1}=0$, ensuring $c_{1}\mu\tau^{2}<1$, together with the assumption about $\mu$ made in this lemma, yields $\mu\in (0,c_{0})$ with
$$
c_{0}=\min\{\frac{1}{32\sqrt{3k\tau^{2}}},\frac{1}{c_{1}\tau^{2}},\frac{-(118406+144k^{2})+\sqrt{(118406+144k^{2})^{2}+18\tau(1+24k^{2})}}{24\tau(1+24k^{2})}\}>0,
$$
which simultaneously satisfies $c_{1}>0$, $c_{1}\mu\tau^{2}<1$, and $\mu\leq\frac{1}{32\sqrt{3k\tau^{2}}}$.
\end{proof}

\subsection{One Epoch Analysis}
We now solve the stochastic recurrence relation for a single epoch\footnote{Note that the proofs of Lemma \ref{lemma.one_epoch}-\ref{lemma.one_epoch2} are a bit different from those in \cite{shamir2015fast}, without similar errors.} of our SVRRG-EIGS algorithm. In this subsection, we still assume that $\max_{l}\|A^{(l)}\|_{2}\leq 1$. Let $\tilde{X}=\tilde{X}^{(s-1)}$, $b_{t}=k-\|V^{\top}X^{(t)}\|_{F}^{2}$ and $\tilde{b}=k-\|V^{\top}\tilde{X}\|_{F}^{2}$ (note that $X^{(0)}=\tilde{X}$). Then by Lemma \ref{lemma14} we have that if $\mu<c_{0}$ and $b_{t-1}\leq \frac{1}{2}$, then
$$
\mathbb{E}[b_{t}|X^{(t-1)}]\leq (1-c_{1}\mu\tau^{2})b_{t-1}+c_{2}\mu^{2}\tau^{2}\tilde{b},
$$
where the expectation is taken with respect to the random $y_{t}$ for $X^{(t)}$.
\begin{lemma} \label{lemma.one_epoch}
Assume $X^{(0)}$ is fixed and $\tilde{b}\leq\frac{1}{2}$. Let $b_{0}=\tilde{b}$ and $E_{t}=\{b_{t'}\leq \frac{1}{2}: t'=0,1,2,\cdots,t\}$. Then
\begin{eqnarray*}
\mathbb{E}[b_{m}|E_{m}]\leq((1-c_{1}\mu\tau^{2})^{m}+\frac{c_{2}}{c_{1}}\mu)\tilde{b}.
\end{eqnarray*}
\end{lemma}

\begin{proof}
We need to examine the evolution of $b_{t}$ as a function of $t$, while $b_{t}$ itself is a deterministic function of $X^{(t)}$ and $X^{(t)}=R_{X^{(t-1)}}(\alpha\tilde{G}(y_{t},X^{(t-1)}))$. Then we have
\begin{eqnarray*}
\mathbb{E}[b_{t}|X^{(t-1)},E_{t}]&=& \mathbb{E}[b_{t}|X^{(t-1)},E_{t-1},b_{t}\leq \frac{1}{2}]\\
&\leq & \mathbb{E}[b_{t}|X^{(t-1)},E_{t-1}]\\
&\leq& (1-c_{1}\mu\tau^{2})[b_{t-1}|X^{(t-1)},E_{t-1}]+c_{2}\mu^{2}\tau^{2}\tilde{b}.
\end{eqnarray*}
Taking expectation over $X^{(t-1)}$ (on behalf of the filtration $\mathcal{F}_{t-1}$) on both sides, unwinding the recursion and noting that $\tilde{b}$ is fixed, we have
\begin{eqnarray*}
\mathbb{E}[b_{t}|E_{t}]&\leq& (1-c_{1}\mu\tau^{2})\mathbb{E}[b_{t-1}|E_{t-1}]+c_{2}\mu^{2}\tau^{2}\tilde{b}\\
&\leq & (1-c_{1}\mu\tau^{2})^{2}\mathbb{E}[b_{t-2}|E_{t-2}]+c_{2}\mu^{2}\tau^{2}\tilde{b}\sum_{i=0}^{1}(1-c_{1}\mu\tau^{2})^{i}\leq\cdots\\
&\leq & (1-c_{1}\mu\tau^{2})^{t}\mathbb{E}[b_{0}|E_{0}]+c_{2}\mu^{2}\tau^{2}\tilde{b}\sum_{i=0}^{t-1}(1-c_{1}\mu\tau^{2})^{i}\\
&= &
(1-c_{1}\mu\tau^{2})^{t}\tilde{b}+c_{2}\mu^{2}\tau^{2}\tilde{b}\sum_{i=0}^{t-1}(1-c_{1}\mu\tau^{2})^{i}\\
&\leq & (1-c_{1}\mu\tau^{2})^{t}\tilde{b}+c_{2}\mu^{2}\tau^{2}\tilde{b}\sum_{i=0}^{\infty}(1-c_{1}\mu\tau^{2})^{i}\\
&=&(1-c_{1}\mu\tau^{2})^{t}\tilde{b}+c_{2}\mu^{2}\tau^{2}\tilde{b}\frac{1}{c_{1}\mu\tau^{2}}=((1-c_{1}\mu\tau^{2})^{t}+\frac{c_{2}}{c_{1}}\mu)\tilde{b}.
\end{eqnarray*}
Setting $t=m$ above completes the proof.

\end{proof}

We now need to show that the event $E_{m}$ occurs w.h.p. so that Lemma \ref{lemma.one_epoch} makes sense in practice.

\begin{lemma} \label{lemma.one_epoch2}
Assume $\mu<c_{0}$. Then for any $\varrho\in(0,1)$ and $m$, if
$$
\tilde{b} + c_{3}km\mu^{2}\tau^{2} + c_{5}k\sqrt{m\mu^{2}\tau^{2}\log(1/\varrho)}\leq \frac{1}{2},
$$
then it holds that the event $E_{m}$ (i.e., $b_{t}\leq \frac{1}{2}$ for all $t=0,1,2,\cdots,m$) occurs with probability at least $1-\varrho$.
\end{lemma}
\begin{proof}
The key here is that $b_{0},b_{1},\cdots,b_{m}$ as a stochastic process induces a super-martingale with respect to the filtration $\mathcal{F}=\{\mathcal{F}_{t}\}$ about random draws $y_{t}$ (note that $X^{(t)}$ is used on behalf of $\mathcal{F}_{t}$ for brevity), and thus is amenable to a concentration of measure argument. In fact, according to the proof of Lemma \ref{lemma14} and noting that
$$
\|X^{\top}VV^{\top}X\|_{F}\leq \|X^{\top}V\|_{2}\|V^{\top}X\|_{F}\leq\|V^{\top}X\|_{F},
$$
we have
\begin{eqnarray*}
\mathbb{E}[b_{t+1}|X^{(t)}]&\leq & b_{t}
 - 2\alpha \tau (\|V^{\top}X\|_{F}^{2} - \|X^{\top}VV^{\top}X\|_{F}^{2})+\\
&& 4\alpha^{2}(1+2\alpha)(k - \|V^{\top}X\|_{F}^{2})+\alpha^{2}(k^{2}(1+2\alpha)+\frac{200}{27}(31+10\kappa_{2}))\kappa_{F}^{2}\\
&\leq & b_{t}+4\alpha^{2}(1+2\alpha)k+192\alpha^{2}(k^{2}(1+2\alpha)+\frac{7400}{9})k\\
&= & b_{t}+c_{3}k\mu^{2}\tau^{2}
\end{eqnarray*}
where $c_{3}=4(1+2\mu\tau)+192(k^{2}(1+2\mu\tau)+\frac{7400}{9})$. Define $\Psi_{t}=b_{t} - c_{3}k\mu^{2}\tau^{2}t$ for $t=0,1,2,\cdots,m$. Note that $0\leq b_{t}\leq k$, and \{$\Psi_{t}:t=0,1,2,\cdots,m$\} is a finite sequence of random variables and thus the natural continuation can be applied to arrive at an infinite sequence such that
$$
|\Psi_{t}|\leq b_{t}+c_{3}k\mu^{2}\tau^{2}m\leq k+c_{3}k\mu^{2}\tau^{2}m
$$
for any $t$ including $t>m$. Meanwhile, we have
\begin{eqnarray*}
\mathbb{E}[\Psi_{t}|X^{(t-1)}]&=& \mathbb{E}[b_{t}|X^{(t-1)}] - c_{3}k\mu^{2}\tau^{2}t\\
&\leq & b_{t-1}+c_{3}k\mu^{2}\tau^{2}- c_{3}k\mu^{2}\tau^{2}t\\
&=& b_{t-1}- c_{3}k\mu^{2}\tau^{2}(t-1) = \Psi_{t-1}.
\end{eqnarray*}
Thus, $\{\Psi_{t}\}$ is a super-martingale. Furthermore, by Lemma \ref{lemma6}, we have
\begin{eqnarray*}
|\Psi_{t+1}-\Psi_{t}|&\leq & |\|V^{\top}X^{(t+1)}\|_{F}^{2}-\|V^{\top}X^{(t)}\|_{F}^{2}|+c_{3}k\mu^{2}\tau^{2}\\
&\leq & \frac{20k\mu\tau}{1-5\mu\tau}+c_{3}k\mu^{2}\tau^{2}\leq \frac{20k\mu\tau}{1-5c_{0}\tau}+c_{3}c_{0}k\mu\tau^{2} =c_{4}k\mu\tau
\end{eqnarray*}
where $c_{4}=\frac{20}{1-5c_{0}\tau}+c_{0}c_{3}\tau$. Now we are able to apply Azuma-Hoeffding inequality and have that for any $t\geq 0$ and $a>0$
\begin{eqnarray*}
P(\Psi_{t}-\Psi_{0}\geq a)&\leq & \exp\{-\frac{a^{2}}{2\sum_{s=1}^{t}(c_{4}k\mu\tau)^{2}}\}=\exp\{-\frac{a^{2}}{2c_{4}tk^{2}\mu^{2}\tau^{2}}\}\\
&\leq & \exp\{-\frac{a^{2}}{2c_{4}mk^{2}\mu^{2}\tau^{2}}\}\triangleq\varrho
\end{eqnarray*}
where $\varrho\in (0,1)$. 
Solving $\varrho=\exp\{-\frac{a^{2}}{2c_{4}mk^{2}\mu^{2}\tau^{2}}\}$ with respect to $a$ yields $a=c_{5}k\sqrt{m\mu^{2}\tau^{2}\log(1/\varrho)}$ with $c_{5}=\sqrt{2c_{4}}$. Therefore, we get that $\Psi_{t}-\Psi_{0}< a$ , i.e.,
$$
b_{t} \leq \tilde{b} + c_{3}k\mu^{2}\tau^{2}t + a \leq  \tilde{b} + c_{3}km\mu^{2}\tau^{2} + c_{5}k\sqrt{m\mu^{2}\tau^{2}\log(1/\varrho)}\leq \frac{1}{2}
$$
for all $t=0,1,2,\cdots,m$, with probability at least $1-\varrho$. 
Note that the condition $\tilde{b} + c_{3}km\mu^{2}\tau^{2} + c_{5}k\sqrt{m\mu^{2}\tau^{2}\log(1/\varrho)}\leq \frac{1}{2}$ implies that $\tilde{b}<\frac{1}{2}$. Then it's reduced to $c_{3}km\mu^{2}\tau^{2} + c_{5}k\sqrt{m\mu^{2}\tau^{2}\log(1/\varrho)}\leq \frac{1}{2}-\tilde{b}$, which can be satisfied by using a sufficiently small $\mu\in(0,c_{0})$ when $\varrho$ is set properly, e.g., $\varrho=\exp\{-1/\mu\}$.

\end{proof}

\begin{lemma} \label{lemma.one_epoch3}
Fix confidence parameters $\varrho,\vartheta\in (0,\frac{1}{2})$ and assume that $\mu,m$ are set such that $\mu<c_{0}$ and
$$
\tilde{b} + c_{3}km\mu^{2}\tau^{2} + c_{5}k\sqrt{m\mu^{2}\tau^{2}\log(1/\varrho)}\leq \frac{1}{2}.
$$
Then it holds that with probability at least $1-(\varrho+\vartheta)$,
$$
b_{m}\leq\frac{1}{\vartheta}((1-c_{1}\mu\tau^{2})^{m}+\frac{c_{2}}{c_{1}}\mu)\tilde{b}.
$$

\begin{proof}
By Markov inequality, we get
\begin{eqnarray*}
P(b_{m}\geq\frac{\mathbb{E}[b_{m}|E_{m}]}{\vartheta})\leq\mathbb{E}[b_{m}|E_{m}]/\frac{\mathbb{E}[b_{m}|E_{m}]}{\vartheta}=\vartheta,
\end{eqnarray*}
while by Lemma \ref{lemma.one_epoch}, we have
\begin{eqnarray*}
\mathbb{E}[b_{m}|E_{m}]\leq((1-c_{1}\mu\tau^{2})^{m}+\frac{c_{2}}{c_{1}}\mu)\tilde{b}.
\end{eqnarray*}
Thus,
$$
P(b_{m}\geq\frac{1}{\vartheta}((1-c_{1}\mu\tau^{2})^{m}+\frac{c_{2}}{c_{1}}\mu)\tilde{b}\;|E_{m})
\leq P(b_{m}\geq\frac{\mathbb{E}[b_{m}|E_{m}]}{\vartheta})\leq \vartheta.
$$
That is, with probability at least $1-\vartheta$,
$b_{m}\leq\frac{1}{\vartheta}((1-c_{1}\mu\tau^{2})^{m}+\frac{c_{2}}{c_{1}}\mu)\tilde{b}$
conditioned on $E_{m}$. Combining with Lemma \ref{lemma.one_epoch2}, we get that
$b_{m}\leq\frac{1}{\vartheta}((1-c_{1}\mu\tau^{2})^{m}+\frac{c_{2}}{c_{1}}\mu)\tilde{b}$
with probability at least $1-(\varrho+\vartheta)$.

\end{proof}

\end{lemma}

\subsection{One Run Analysis}
We now proceed to the analysis on one complete run of our SVRRG-EIGS algorithm. Again, assume that $\max_{l}\|A^{(l)}\|_{2}\leq 1$. 
Let $\tilde{b}_{s}=k-\|V^{\top}\tilde{X}^{(s)}\|_{F}^{2}$ and assume that $\tilde{b}_{0}< \frac{1}{2}$. Then by Lemma \ref{lemma.one_epoch3}, for any $\varrho,\vartheta\in(0,\frac{1}{2})$, $\mu\in(0,\min\{c_{0},\frac{c_{1}}{2c_{2}}\vartheta^{2}\})$, and $m\geq \frac{3\log(1/\vartheta)}{c_{1}\mu\tau^{2}}$ such that
$$
\tilde{b}_{0} + c_{3}km\mu^{2}\tau^{2} + c_{5}k\sqrt{m\mu^{2}\tau^{2}\log(1/\varrho)}\leq \frac{1}{2},
$$
we have that
\begin{eqnarray*}
b_{m}&\leq&\frac{1}{\vartheta}((1-c_{1}\mu\tau^{2})^{m}+\frac{c_{2}}{c_{1}}\mu)\tilde{b}_{0}\\
&\leq&\frac{1}{\vartheta}((1-c_{1}\mu\tau^{2})^{\frac{3\log(1/\vartheta)}{c_{1}\mu\tau^{2}}}+\frac{c_{2}}{c_{1}}\frac{c_{1}}{2c_{2}}\vartheta^{2})\tilde{b}_{0}
\end{eqnarray*}
with probability at least $1-(\varrho+\vartheta)$. Note that $1+x\leq\exp\{x\}$ for any $x$ and hence $\log(1-x)\leq -x$ for any $0<x<1$ which in turn induces $\frac{a}{x}\log(1-x)\leq -a$, i.e., $(1-x)^{\frac{a}{x}}\leq\exp\{-a\}$, for any $0<x<1$ and $a>0$. Since $0<c_{1}\mu\tau^{2}<1$ by $\mu<c_{0}$, we can write
\begin{eqnarray*}
b_{m}\leq\frac{1}{\vartheta}((1-c_{1}\mu\tau^{2})^{\frac{3\log(1/\vartheta)}{c_{1}\mu\tau^{2}}}+\frac{c_{2}}{c_{1}}\frac{c_{1}}{2c_{2}}\vartheta^{2})\tilde{b}_{0}
\leq\frac{1}{\vartheta}(\vartheta^{3}+\frac{1}{2}\vartheta^{2})\tilde{b}_{0}=\vartheta(\vartheta+\frac{1}{2})\tilde{b}_{0}\leq \vartheta\tilde{b}_{0}\leq\frac{1}{2}.
\end{eqnarray*}
Noting that $\tilde{b}_{1}=b_{m}$, we get $\tilde{b}_{1}\leq\vartheta\tilde{b}_{0}\leq\tilde{b}_{0}$ with probability at least $1-(\varrho+\vartheta)$. In a similar fashion, since $\tilde{b}_{1}\leq\tilde{b}_{0}$ and thus
$$
\tilde{b}_{1} + c_{3}km\mu^{2}\tau^{2} + c_{5}k\sqrt{m\mu^{2}\tau^{2}\log(1/\varrho)}\leq \frac{1}{2},
$$
we can apply Lemma \ref{lemma.one_epoch3} on the second epoch and get $\tilde{b}_{2}\leq\vartheta\tilde{b}_{1}\leq\vartheta^{2}\tilde{b}_{0}$ with probability at least $1-(\varrho+\vartheta)$, conditioned on the first epoch. If conditioned on the initial setting, we then have $\tilde{b}_{2}\leq\vartheta^{2}\tilde{b}_{0}$ with probability at least $1-2(\varrho+\vartheta)$ provided that $\varrho,\vartheta\in (0,\frac{1}{4})$. In this way, we can see that repeating above process till the $T$-th epoch will result in
$$
k-\|V^{\top}\tilde{X}^{(T)}\|_{F}^{2} = \tilde{b}_{T}\leq\vartheta^{T}\tilde{b}_{0}< \vartheta^{T}
$$
with probability at least $1-T(\varrho+\vartheta)$, conditioned on the initial setting and $\varrho,\vartheta\in (0,\frac{1}{2T})$. Then solving $\vartheta^{T}\leq \varepsilon$ for $T$ tells that $T=\lceil\frac{\log(1/\varepsilon)}{\log(1/\vartheta)}\rceil$ epochs suffice to achieve any aimed accuracy $\varepsilon\in (0,1)$ for $k-\|V^{\top}\tilde{X}^{(T)}\|_{F}^{2}\leq\varepsilon$ with probability at least $1-\lceil\frac{\log(1/\varepsilon)}{\log(1/\vartheta)}\rceil(\varrho+\vartheta)$.

To simplify these expressions, setting $\varrho=\vartheta=\frac{\varphi}{2}$ leads to
$$
\lceil\frac{\log(1/\varepsilon)}{\log(1/\vartheta)}\rceil(\varrho+\vartheta)
=\lceil\frac{\log(1/\varepsilon)}{\log(2/\varphi)}\rceil\varphi\leq \lceil\frac{\log(1/\varepsilon)}{\log(2)}\rceil\varphi=\lceil\log_{2}(1/\varepsilon)\rceil\varphi.
$$
Accordingly, the initial conditions become
$\varphi\in (0,\frac{1}{\lceil\log_{2}(1/\varepsilon)\rceil})$, $\mu\in(0,\min\{c_{0},\frac{c_{1}}{8c_{2}}\varphi^{2}\})$, $m\geq \frac{3\log(2/\varphi)}{c_{1}\mu\tau^{2}}$ and
$$
\tilde{b}_{0} + c_{3}km\mu^{2}\tau^{2} + c_{5}k\sqrt{m\mu^{2}\tau^{2}\log(2/\varphi)}\leq \frac{1}{2}.
$$
With the assumption $\tilde{b}_{0}<\frac{1}{2}$, we could rewrite the above inequality as
$$
c_{3}km\mu^{2}\tau^{2} + c_{5}k\sqrt{m\mu^{2}\tau^{2}\log(2/\varphi)}\leq \frac{1}{2}-\tilde{b}_{0}.
$$
Now we can conclude that, for any $\varepsilon\in (0,1)$ and any $\varphi\in (0,\frac{1}{\lceil\log_{2}(1/\varepsilon)\rceil})$, we have $k-\|V^{\top}\tilde{X}^{(T)}\|_{F}^{2}\leq\varepsilon$ with probability at least $1-\lceil\log_{2}(1/\varepsilon)\rceil\varphi$ by running $T=\lceil\frac{\log(1/\varepsilon)}{\log(2/\varphi)}\rceil$ epochs of our SVRRG-EIGS algorithm, if the following conditions are satisfied:
\begin{eqnarray*}
&&\max_{l}\|A^{(l)}\|_{2}\leq 1,\quad \tilde{b}_{0}<\frac{1}{2},\quad \alpha\in(0,\min\{c_{0}\tau,\frac{c_{1}}{8c_{2}}\tau\varphi^{2}\}),\\
&&\quad m\geq \frac{3\log(2/\varphi)}{c_{1}\alpha\tau},\quad c_{3}km\alpha^{2} + c_{5}k\sqrt{m\alpha^{2}\log(2/\varphi)}\leq \frac{1}{2}-\tilde{b}_{0},
\end{eqnarray*}
where the positive constants are
\begin{eqnarray*}
c_{0}&=&\min\{\frac{1}{32\sqrt{3k\tau^{2}}},\frac{1}{c_{1}\tau^{2}},\frac{-(118406+144k^{2})+\sqrt{(118406+144k^{2})^{2}+18\tau(1+24k^{2})}}{24\tau(1+24k^{2})}\},\\
c_{1}&=& \frac{2}{\tau}(\frac{1}{8}\tau-2\alpha(1+2\alpha)(1+24k^{2})-\frac{118400}{3}\alpha),\quad c_{2} = 96(k^{2}(1+2\alpha)+823),\\
c_{3}&=& 4(1+2\alpha)+192(k^{2}(1+2\alpha)+\frac{7400}{9}),\quad c_{4}=\frac{20}{1-5c_{0}\tau}+c_{0}c_{3}\tau,\quad c_{5}=\sqrt{2c_{4}}.
\end{eqnarray*}

\bibliographystyle{plain}
\bibliography{xuzq16}

\end{document}